%% file: neurips/neurips_2026.tex
\documentclass{article}

\usepackage[preprint]{neurips/neurips_2026}

\usepackage[utf8]{inputenc} 
\usepackage[T1]{fontenc}    
\usepackage{hyperref}       
\usepackage{url}            
\usepackage{booktabs}       
\usepackage{amsfonts}       
\usepackage{nicefrac}       
\usepackage[patch=none]{microtype}
\usepackage{xcolor}         

\usepackage{amsmath}
\usepackage{amssymb}
\usepackage{mathtools}
\usepackage{amsthm}

\usepackage{multirow}

\usepackage[capitalize]{cleveref}

\usepackage{subcaption}
\input{neurips/macros}

\theoremstyle{plain}
\newtheorem{theorem}{Theorem}[section]
\usepackage[textsize=tiny]{todonotes}
\ifdefined\showcomments
\newcommand{\amirg}[1]{\textcolor{red}{AG:{#1}}}
\newcommand{\ym}[1]{\textcolor{blue}{YM:{#1}}}
\newcommand{\cm}[1]{\textcolor{red}{CM:{#1}}}
\newcommand{\hk}[1]{\textcolor{red}{HK: {#1}}}
\newcommand{\tk}[1]{\textcolor{magenta}{TK: {#1}}}
\else
  \newcommand{\amirg}[1]{}
  \newcommand{\cm}[1]{}
  \newcommand{\ym}[1]{}
  \newcommand{\hk}[1]{}
  \newcommand{\tk}[1]{}
\fi

\hypersetup{
  breaklinks   = true, 
  colorlinks   = true, 
  urlcolor     = blue, 
  linkcolor    = blue, 
  citecolor    = blue 
}

\title{Cost-Aware Learning}

\author{%
  Clara Mohri\thanks{Work done while Clara Mohri was a Student Researcher at Google Research.  Correspondence to: \texttt{cmohri@harvard.edu}.}\\
  Kempner Institute \\
  Harvard University \\
  \And
  Amir Globerson \\
  Google Research \\
  Tel Aviv University \\
  \And
  Haim Kaplan \\
  Google Research \\
  Tel Aviv University \\
  \AND
  Tomer Koren \\
  Google Research \\
  Tel Aviv University \\
  \And
  Yishay Mansour \\
  Google Research \\
  Tel Aviv University \\
}

\begin{document}

\maketitle

\begin{abstract}
We consider the problem of Cost-Aware Learning, where sampling different components of a finite-sum objective incurs different costs. The objective is to reach a target error while minimizing the total cost. We propose Cost-Aware SGD, which uses a distribution based on gradient norms and costs to sample components. We provide a thorough analysis of this algorithm, including cost-improvement bounds over baselines, a characterization of distribution proxy sub-optimality, and a  lower bound. We apply our theoretical insights to reinforcement learning with language models, where the computational cost of sequence-level policy gradients varies with length. We find that the advantage magnitude serves as a high-fidelity proxy for gradient norms, and use this to introduce Cost-Aware GRPO. Empirical results on 1.5B, 4B, and 8B LLMs demonstrate that this algorithm significantly reduces the tokens used in policy optimization while matching or exceeding baseline accuracy.
\end{abstract}

\input{neurips/sections/intro}
\input{neurips/sections/related}
\input{neurips/sections/theory}

\input{neurips/sections/experiments}

\newpage
\bibliography{neurips/bibliography}
\bibliographystyle{plainnat}


\appendix
\onecolumn
\clearpage
\onecolumn

\makeatletter
\def\addcontentsline#1#2#3{%
  \addtocontents{#1}{\protect\contentsline{#2}{#3}{\thepage}{page.\thepage}}%
}

\let\old@sect\@sect
\def\@sect#1#2#3#4#5#6[#7]#8{%
  \phantomsection 
  \old@sect{#1}{#2}{#3}{#4}{#5}{#6}[#7]{#8}%
}

\let\old@ssect\@ssect
\def\@ssect#1#2#3#4#5{%
  \phantomsection
  \old@ssect{#1}{#2}{#3}{#4}{#5}%
}
\makeatother

\setcounter{tocdepth}{2}

\begingroup
\setlength{\parskip}{0pt}
\setlength{\baselineskip}{0pt}
\renewcommand{\baselinestretch}{1.0}
\tableofcontents
\endgroup
\newpage

\input{neurips/appendix/related_work}
\input{neurips/appendix/algorithms}
\input{neurips/appendix/subset_selection}
\input{neurips/appendix/proofs}

\input{neurips/appendix/lower_bound}
\input{neurips/appendix/theory_vs_practice}

\input{neurips/appendix/experiment_details}



\end{document}

%% file: neurips/macros.tex
\usepackage{amssymb}
\usepackage{amsmath}
\usepackage{latexsym}
\usepackage{graphicx}
\usepackage{color}
\usepackage{graphics}
\usepackage{algorithm}
\usepackage{algorithmic}
\usepackage{natbib}
\usepackage{bbm}
\usepackage{cancel}
\usepackage{setspace}
\usepackage{makeidx}
\usepackage[mathscr]{euscript}
\usepackage{mathtools}

\usepackage[vvarbb]{newtxmath}

\DeclarePairedDelimiter{\abs}{\lvert}{\rvert}

\DeclarePairedDelimiter{\paren}{(}{)}
\DeclarePairedDelimiter{\norm}{\|}{\|}

\DeclareMathOperator*{\E}{\mathbb{E}}

\DeclareMathOperator{\Var}{Var}
\DeclareMathOperator{\Cov}{Cov}

\providecommand{\abs}[1]{\lvert#1\rvert}

\providecommand{\norm}[1]{\lVert#1\rVert}

\newcommand{\cD}{\mathcal{D}}

\newcommand{\cK}{\mathcal{K}}

\newcommand{\e}{\epsilon}

\newcommand{\ignore}[1]{}

\newenvironment{proof*}{\trivlist\item[\hskip\labelsep{\it Proof}{.}]}

\newcommand{\sfp}{{\mathsf p}}

\newcommand{\sfq}{{\mathsf q}}

%% file: neurips/sections/intro.tex
\section{Introduction}

Stochastic gradient descent (SGD)~\citep{nemirovsky1983wiley} theory typically measures convergence in terms of the number of gradient steps required to reach a target error. This implicitly   assumes that gradient evaluation on any data point incurs an identical computational cost. However, in modern language model training, the assumption of uniform cost is not always true. For example, in reinforcement learning for language models~\citep{ouyang2022training, shao2024deepseekmath}, the FLOPs required to compute a policy gradient scale linearly with sequence length, meaning a long reasoning trace can cost an order of magnitude more than a concise response.

Motivated by this computational heterogeneity, we study the problem of \emph{Cost-Aware Learning}: given a finite-sum objective where each component has a known evaluation cost $c_i$, how should one sample gradients to reach $\e$-error at a minimum total cost?  We propose Cost-Aware SGD, which leverages importance sampling \citep{kloek1978bayesian, beygelzimer2009importance} to derive the optimal sampling distribution $p_i^* \propto G_i/\sqrt{c_i}$, where $G_i$ is the Lipschitz constant (gradient norm bound) of component $i$. 

We analyze the cost-improvement over uniform and variance-reduction baselines, and validate this in a synthetic setting. Since exact gradient norms are often unavailable in practice, we characterize the sub-optimality of using a proxy distribution. We establish an information-theoretic lower bound for cost-aware learning, which motivates a subset selection algorithm. Finally, we show that the derived sampling strategy extends beyond convex objectives. 


We apply these insights to the policy optimization phase of GRPO \citep{shao2024deepseekmath}, where sequence length provides a natural cost. A key empirical finding, motivated by the form of the policy gradient, is that the advantage magnitude $\abs{A_i}$ serves as a high-fidelity proxy for $G_i$ based on our sub-optimality metrics. Equipped with this proxy, we introduce Cost-Aware GRPO, which samples model-generated sequences according to $p^*$ with $\abs{A_i}$ in place of $G_i$ during policy optimization. Cost-Aware GRPO significantly reduces the FLOPs used in policy optimization while matching or exceeding baseline performance. Through extensive ablations, we demonstrate robustness to proxy noise, hyperparameter choices, and show that cost-aware sampling generalizes beyond GRPO to the CISPO objective~\citep{chen2025minimax}.

The main contributions of this paper are as follows:
\begin{enumerate}
    \item \textbf{Cost-Aware SGD (\S\ref{sec:importance_sampling}). }We derive the optimal cost-aware sampling distribution for convex and strongly convex finite-sum optimization, analyze its cost-improvement over baselines, characterize the sub-optimality of using a proxy distribution, and establish an information-theoretic lower bound. 
    \item \textbf{Cost-Aware GRPO (\S\ref{sec:ca_grpo}). }We identify $\abs{A_i}$ as a high-fidelity proxy for gradient norms and apply cost-aware sampling to GRPO. Experiments with Qwen2.5-Math-1.5B-Instruct~\citep{yang2024qwen2}, Qwen3-4B Base and Qwen3-8B Base~\citep{yang2025qwen3} demonstrate consistent improvements in token efficiency  without compromising on performance. 
\end{enumerate}

\begin{figure}
    \centering
    \includegraphics[width=0.85\linewidth]{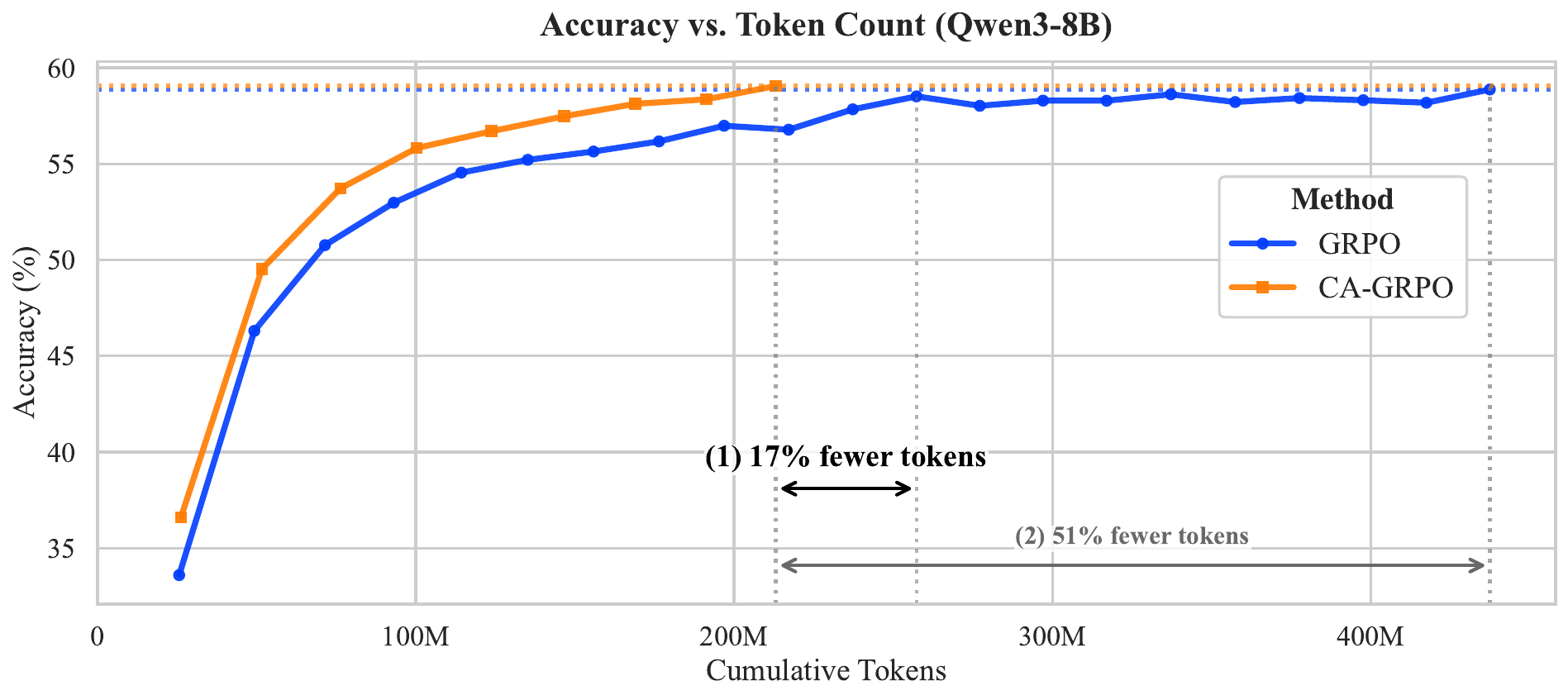}
    \caption{\textbf{Qwen3-8B training with GRPO and Cost-Aware GRPO (CA-GRPO).} We evaluate on AIME1983-2024 (\texttt{pass@1/mean@32}) and report cumulative tokens used in policy gradient computation. We plot every $20$ training steps until peak accuracy is attained for both. CA-GRPO requires: (1) 17\% fewer tokens to reach its absolute peak than GRPO requires to come within 2\% of that performance, (2)  51\% fewer tokens to match or exceed the peak accuracy of GRPO. }
    \label{fig:qwen3-8b-results}
\end{figure}

%% file: neurips/sections/related.tex
\section{Related Work}

We describe the primary related work, and defer the discussion of additional works to Appendix~\ref{app:related_work}.

\paragraph{Importance sampling for SGD.} A significant body of work focuses on accelerating SGD convergence via importance sampling, primarily through variance reduction~\citep{zhao2015stochastic, needell2014stochastic}. A foundational result by \citet{alain2015variance} established that the optimal sampling distribution for variance reduction assigns probabilities proportional to the per-sample gradient norms. In the context of deep learning, where exact gradient norms are expensive to compute, methods have been proposed to approximate these norms using upper bounds or loss values~\citep{katharopoulos2018not, johnson2018training}. However, these approaches operate under a standard optimization framework that treats iterations as the unit of resource consumption. They aim to maximize the information gain \emph{per step}, implicitly assuming uniform computational cost across samples. Our work generalizes this by considering the \emph{cost-aware} setting where the aim is to minimize training cost, leading to a sampling distribution that balances gradient magnitude against sample cost.

In particular, the work of \citet{pmlr-v206-chen23i} shares similarities with our work by also studying a finite-sum setting with varying sample costs. However, our contributions differ significantly in analysis and scope: while their work focuses on asymptotic behavior, we provide finite-time guarantees complemented by a lower bound. Furthermore, we implement this cost-aware framework with language models via Cost-Aware GRPO.

\paragraph{RL for Language Models.} 
While policy gradient methods \citep{NIPS1999_464d828b, kakade2001natural, schulman2015trust} have been used in many settings, the works of \citet{ouyang2022training, dai2023safe} have garnered significant attention for their application to LLM post-training. Various methods such as DPO \citep{rafailov2023direct}, PPO \citep{schulman2017proximal, zheng2023secrets}, and GRPO \citep{shao2024deepseekmath} have been shown to be extremely effective in improving accuracy on standard benchmarks. However, these methods incur substantial computational and memory costs. Several methods aim to improve data efficiency: \citep{wang2025reinforcement, fatemi2025concise, li2025limr, yu2025dapo, xu2025not, zheng2025act}.
\citet{yu2025dapo} introduces, among other contributions, a mechanism to continue resampling from a model until a sufficient amount of non-zero advantage responses have been generated, showing that this data helps the model achieve better performance. 
\citet{xu2025not} attempts to subsample the generations of a model in an effort to train on less data. \citet{zheng2025act} keeps track of the history of various prompts in an effort to remove low-signal prompts. Our work contrasts with these by using importance reweighting and taking into account the computational cost of sequences when making decisions.

%% file: neurips/sections/theory.tex
\section{Problem Formulation} 
\label{sec:problem_formulation}

We consider minimizing a convex function $f$ in the finite-sum setting where $f(x) = \frac{1}{n} \sum_{i=1}^n f_i(x)$. We assume that $f$ is minimized at $x^*$, and that each component $f_i$ is a convex and $G_i$-Lipschitz function. In the cost-aware setting, each component $i$ has an evaluation cost $c_i \geq 0$ for querying its gradient oracle.
Namely, upon querying a particular sample $i$ at a point $x$, the algorithm obtains the exact gradient $\nabla f_i(x)$ and incurs cost $c_i$. 
We assume that the learning algorithm has access to $G_i$ and $c_i$ for all samples $i$. 

For a given $\e > 0$, the objective of a (possibly randomized) cost-aware learning algorithm is to incur the minimum possible cost to obtain $\hat{x}$ such that $\E[f(\hat{x}) - f(x^*)] \leq \e$.

\section{Cost-Aware SGD}
\label{sec:importance_sampling}

In this section, we address the problem presented in Section~\ref{sec:problem_formulation}. We propose Cost-Aware SGD,  a simple variant of standard SGD, which differs only in the use of a non-uniform sampling distribution and importance weighting to account for bias. Pseudocode
is provided in 
Appendix~\ref{alg:cost_weighted_sgd}. While the analysis in this section assumes convex loss functions in order to provide meaningful insights, in Appendix~\ref{app:extensions} we show that the optimal sampling distribution derived does not strictly rely on convexity.

\subsection{Cost Analysis}
\label{subsec:cost_analysis}
Consider a fixed sampling distribution $\sfp \in \Delta_n$ used in the Cost-Aware SGD algorithm. The total cost of the algorithm over $T$ iterations is the random variable $\mathcal{K}_T = \sum_{t=1}^T c_{i_t}$. By the linearity of expectation, the expected total cost is:
\begin{equation}
  \E[\mathcal{K}_T] = \sum_{t=1}^T \E_{i_t \sim \sfp}[c_{i_t}] = T \sum_{i=1}^n p_i c_i = T \cdot C(\sfp),
\end{equation}
where $C(\sfp) = \sum_{i=1}^n p_i c_i$ represents the expected cost per iteration. Our goal is to minimize the total expected cost required to reach a target error $\e$. Let $T(\e, \sfp)$ denote the number of iterations required to achieve expected error $\e$ using distribution $\sfp$. We seek to find $p$ which minimizes:
\begin{equation}
    \cK(\e, \sfp) = T(\e, \sfp) \cdot C(\sfp).
\end{equation}

\paragraph{Case 1: General Convex Functions}
For general convex functions, standard convergence results for SGD \citep{nemirovsky1983wiley, bubeck2015convex} state that with step size $\eta \propto 1/\sqrt{T}$, the expected suboptimality satisfies: $\E[f(\bar{x}_T) - f(x^*)] \le \frac{D \sqrt{S(\sfp)}}{\sqrt{T}}$,
where $S(\sfp) = \E_{i \sim \sfp}[\|\tilde{g}\|^2]$ is the second moment of the gradient estimator, and $D$ is the diameter of the parameter space. Expanding the second moment for importance sampling:
  $S(\sfp) = \sum_{i=1}^n p_i \norm*{ \frac{\nabla f_i(x)}{n p_i} }^2 \leq \frac{1}{n^2} \sum_{i=1}^n \frac{G_i^2}{p_i},$
where $G_i$ is the Lipschitz constant for component $f_i$.
To guarantee error $\e$, we require $T = \frac{D^2 S(\sfp)}{\e^2}$. Substituting this into the cost formula:
\begin{equation}
\label{eq:cost_to_epsilon}
  \cK(\e, \sfp) = \paren*{ \frac{D^2 S(\sfp)}{\e^2} } C(\sfp) \propto S(\sfp) C(\sfp).
\end{equation}

\paragraph{Case 2: Strongly Convex Functions}
For $f$ that is $\mu$-strongly convex, SGD with decaying step size $\eta_t = 1/(\mu t)$ achieves a rate of $\mathcal{O}(S(\sfp) / \mu T)$ \citep{rakhlin2011making}.
The expected total cost under distribution $\sfp$ to attain error $\e$ becomes:
\begin{equation}
\label{eq:strongly_convex_cost_to_e}
  \cK(\e, \sfp) = \paren*{ \frac{4 S(\sfp)}{\mu \e} } C(\sfp) \propto S(\sfp) C(\sfp).
\end{equation}
In both cases, minimizing the total cost is equivalent to minimizing the product $J(\sfp) = S(\sfp)C(\sfp)$.

\subsection{Derivation of the Optimal Distribution}

\begin{theorem}[Optimal Cost-Weighted Distribution]
\label{thm:p_star}
The sampling distribution $\sfp^*$ that minimizes $\cK(\epsilon,p)$ is: 
\begin{equation}
\label{eq:p_star_opt}
    p_i^* = \frac{G_i / \sqrt{c_i}}{\sum_{j=1}^n G_j / \sqrt{c_j}}.
\end{equation}
Under this distribution, the expected costs for general convex and $\mu$-strongly convex functions are:
\begin{equation}
\cK_{\text{cvx}}(\e, \sfp^*) = \frac{D^2}{\e^2} \paren*{ \frac1n \sum_{i=1}^n G_i \sqrt{c_i} }^2,
\quad \quad \quad
\cK_{\text{str-cvx}}(\e, \sfp^*) = \frac{4}{\mu \e} \paren*{ \frac1n \sum_{i=1}^n G_i \sqrt{c_i} }^2.
\end{equation}
\end{theorem}

\begin{proof}
We  minimize the product term $J(\sfp) = S(\sfp)C(\sfp)$.
By the Cauchy-Schwarz inequality, for vectors $u,v\in\mathbb{R}^n$ defined by $u_i = \sqrt{p_i c_i}$ and $v_i = G_i / \sqrt{p_i}$:
\begin{equation}
  \paren*{\sum_{i=1}^n p_i c_i} \paren*{\sum_{i=1}^n \frac{G_i^2}{p_i}}
  = \|u\|^2 \|v\|^2 \geq \langle u, v \rangle^2
  = \paren*{ \sum_{i=1}^n G_i \sqrt{c_i} }^2.
\end{equation}
Equality holds if and only if $u \propto v$, which implies $\sqrt{p_i c_i} \propto G_i / \sqrt{p_i}$, or $p_i \propto G_i / \sqrt{c_i}$.
Substituting the minimum value of the product back into $J(\sfp)$, we obtain $J(\sfp^*) = \frac{1}{n^2} (\sum G_i \sqrt{c_i})^2$.
Substituting $J(\sfp^*)$ into Eq.~\ref{eq:cost_to_epsilon} and Eq.~\ref{eq:strongly_convex_cost_to_e} yields the stated costs.
\end{proof}

\subsection{Cost-Improvement over Baselines}

We compare $\sfp^*$ against two baselines. \textbf{Uniform sampling} sets $(p_\text{unif})_i = 1/n$. The \textbf{variance strategy} ignores costs and samples proportionally to gradient norms: $(p_\text{var})_i = G_i / \sum_j G_j$, which minimizes the variance of the gradient estimator \citep{alain2015variance, zhao2015stochastic, needell2014stochastic} (see Appendix~\ref{app:proof_costblind} for a derivation).

Their expected costs to achieve $\epsilon$ error under general convex losses are:
$\cK(\e, p_\text{unif}) = \frac{D^2}{\e^2n^2} \paren*{\sum_{i = 1}^n G_i^2} \paren*{\sum_{i = 1}^n c_i}$ and $\cK(\e, p_\text{var}) = \frac{D^2}{\e^2n^2} \paren*{\sum_{i = 1}^n G_i} \paren*{\sum_{i = 1}^n G_i c_i}.$   The cost ratios relative to optimal sampling are:
\begin{equation}
    \frac{\cK(\e, \sfp^*)}{\cK_(\e, p_\text{unif})} = \frac{(\sum G_i \sqrt{c_i})^2}{(\sum G_i^2)(\sum c_i)} \le 1, \quad \quad \quad 
    \frac{\cK(\e, \sfp^*)}{\cK_{\text{var}}(\e,p_\text{var} )} = \frac{(\sum G_i \sqrt{c_i})^2}{(\sum G_i) \paren*{\sum_i G_i c_i}} \le 1.
\end{equation}
When all costs are equal, $\cK_\text{opt} = \cK_\text{var}$, but $\cK_\text{opt}$ still improves over $\cK_\text{unif}$ by minimizing variance (see Appendix~\ref{app:ada_unif_comp} for a detailed comparison).  When costs and gradients are negatively correlated, the gains can be dramatic: setting $G_i = 2^{-i/4}$, $c_i = 2^i$ yields $\cK_\text{opt}/\cK_\text{unif} = \Theta(2^{-n/2})$ and $\cK_\text{opt}/\cK_\text{var} = \Theta(2^{-n/4})$. 

\subsection{Synthetic Validation}
\label{subsec:synthetic_unbiased}
We evaluate the sampling strategies on a linear least squares task with $N=3000$ points in dimension $d=50$. The data vectors $a_i$ are scaled such that their norms are bounded by $L=10$. This geometry ensures that the Lipschitz constants $G_i$ are bounded but heterogeneous across the dataset. 
\begin{figure}
    \centering
    \includegraphics[width=0.9 \linewidth]{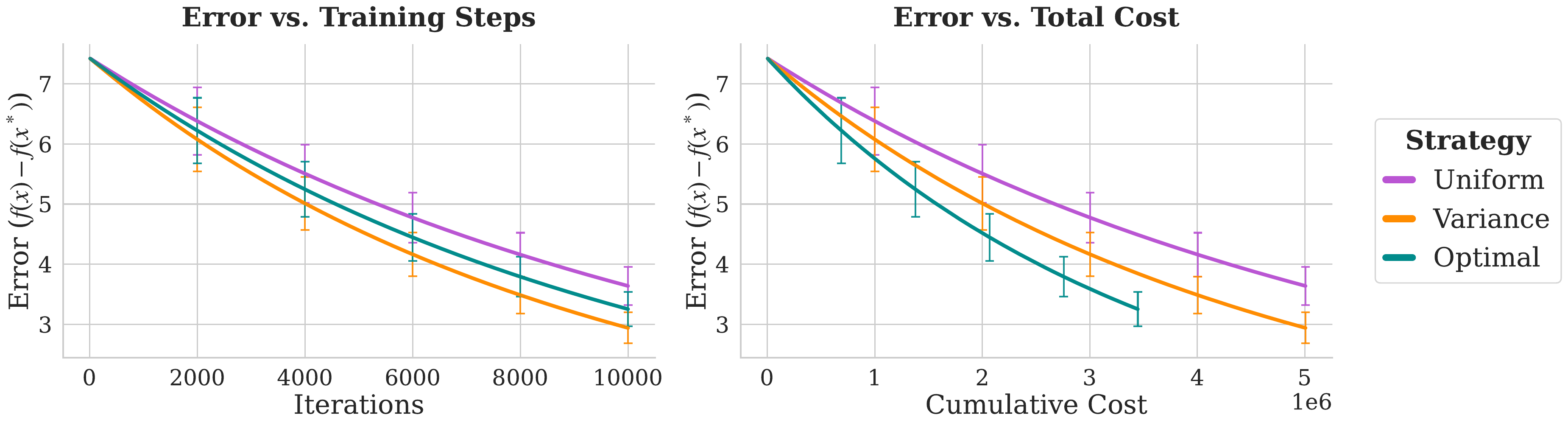}
    \caption{\textbf{Synthetic validation. } We compare the error with the total training steps and the total cost incurred by each sampling strategy. }
\label{fig:comparison_synthetic}
\end{figure}
Evaluation costs $c_i$ are assigned randomly from a uniform distribution between $1$ and $1000$.  We compare three sampling approaches: \textit{Uniform}, \textit{Variance} ($p_i \propto G_i$), and \textit{Optimal} ($p_i \propto G_i / \sqrt{c_i}$). The results are shown in Figure~\ref{fig:comparison_synthetic}. We average over $1000$ trials and report standard deviation error bars.

\textbf{Cost Efficiency:} For a fixed error target, the \textit{Optimal} strategy incurs the least total cost. By weighting samples by the ratio of information to cost, it effectively identifies samples that provide significant variance reduction with a low cost. \textbf{Convergence Rate:} The \textit{Variance} strategy converges the fastest in terms of iteration count. This is expected, as it minimizes the gradient estimator's variance without regard for computational budget.


\subsection{Proxy Sub-Optimality}
\label{subsec:subopt}
In practice, it can be prohibitively costly to compute $G_i$.  As such, one may not have access to the perfect sampling distribution $p^*$; in fact, this is the case for later experiments in this paper. In this subsection, we consider the sub-optimality of using a proxy. We recall that $J(\sfp) = S(\sfp)C(\sfp)$, and that $\cK(\e, \sfp)$ is the expected cost to attain error $\e$ under distribution $\sfp$. In Section~\ref{subsec:cost_analysis}, we showed that for any $\e$,  $\cK(\e, \sfp) \propto J(\sfp)$. $J(p')$ represents how cost scales with error for $p'$. Hence, the quantity $J(p')/J(p^*)$ represents the ratio between the cost incurred by $p'$ and that of $p^*$.

The $\chi^2$-divergence is defined as follows for two distributions $P, Q$: $D_{\chi^2}(P || Q)=\sum \frac{P^2}{Q} -1 $. In Theorem~\ref{thm:distribution_subopt}, we exactly characterize the sub-optimality due to using a different sampling distribution $p'$ in place of $p^*$.

\begin{theorem}[Sub-optimality gap for distribution proxy] 
\label{thm:distribution_subopt}
Let $p'$ be any empirical sampling distribution. We define the cost-biased distribution $\tilde{p}$ corresponding to $p$ as:
$\tilde{p}_i = (p_i c_i)/C(p).$
The sub-optimality gap is:
$J(p')/J(p^*) = 1 + D_{\chi^2}(\tilde{p}^* || \tilde{p}').$
\end{theorem}

To provide further intuition as to when one can use some proxy $G_i'$ in place of $G_i$ for a proxy sampling distribution, we provide Theorem~\ref{thm:proxy_subopt} (Appendix~\ref{app:proofs_subopt}). It implies that, when $G_i$ and $G_i'$ have near-$1$ Pearson correlation, $\E[J(p')] \approx J(p^*)$. The Pearson correlation between random variables $X$ and $Y$ with expected values $\mu_X$, $\mu_Y$, and standard deviations $\sigma_X$, $\sigma_Y$, respectively, is defined as:
$\rho_{X,Y} = \frac{\text{Cov}(X,Y)}{\sigma_X \sigma_Y}.$ Both proofs are deferred to Appendix~\ref{app:proofs_subopt}.

\subsection{Lower Bound}
\label{sec:lower_bounds}

Having established the upper bounds for Cost-Aware SGD, a
natural question is whether it is possible to design a sampling strategy
that achieves a lower total cost. In this section, we provide a lower
bound for the class of convex functions.  


\begin{theorem} \label{thm:lb}
Fix $\epsilon, G > 0$. For any number of components $n \ge G^2/\epsilon^2$, any set of nonnegative query costs $\{c_i\}_{i=1}^n$, and any (possibly randomized) algorithm with an expected error guarantee $\mathbb{E}[F(\hat{x}) - \min F(x)] \le \epsilon$, there exists a convex, $G$-Lipschitz (i.e., with $G_i=G$ for all $i$) finite-sum problem instance over $\mathcal{X} = [-1,1]$ such that the algorithm's expected total query cost is at least:
\begin{align*}
    \Omega\left( \frac{G^2}{\epsilon^2} \left( \frac{1}{n} \sum_{i \in S^*} \sqrt{c_i} \right)^2 \right),
\end{align*}
where $S^* \subseteq [n]$ is any set of components such that following cost-uniformity condition holds:
$(\max_{i \in S^*} \sqrt{c_i})/(\tfrac1n \sum_{j \in S^*} \sqrt{c_j}) \le \frac{G}{40\epsilon}.$
\end{theorem}

We defer the proof to Appendix~\ref{app:lb}. The lower bound matches the upper bound from \cref{thm:p_star} in its $\frac{1}{n^2\epsilon^2}$ dependence, but differs in how it aggregates costs: the upper bound scales with $(\sum_{i=1}^n \sqrt{c_i})^2$, while the lower bound scales with $(\sum_{i \in S^*} \sqrt{c_i})^2$. For uniform costs, $S^* = [n]$ and the bounds match up to constants. When costs are highly non-uniform, the gap arises because expensive components are excluded from $S^*$. This motivates combining importance-weighted sampling with \emph{subset selection}: first choosing which components to include, then optimizing sampling probabilities over the selected subset. We develop this approach in Appendix~\ref{app:subset_selection}, where we show that the subset selection problem reduces to a min-cost knapsack problem, for which a polynomial-time 2-approximation exists that simply selects the cheapest components.


%% file: neurips/sections/experiments.tex
\section{Cost-Aware GRPO}
\label{sec:ca_grpo}
We apply our theoretical analysis to the empirical setting of GRPO \citep{shao2024deepseekmath}. In particular, we focus on \emph{reducing the number of FLOPs used by policy optimization}. 

\begin{figure}[t]
    \centering
    
    \begin{subfigure}[b]{0.48\linewidth}
        \centering
        \includegraphics[width=\linewidth]{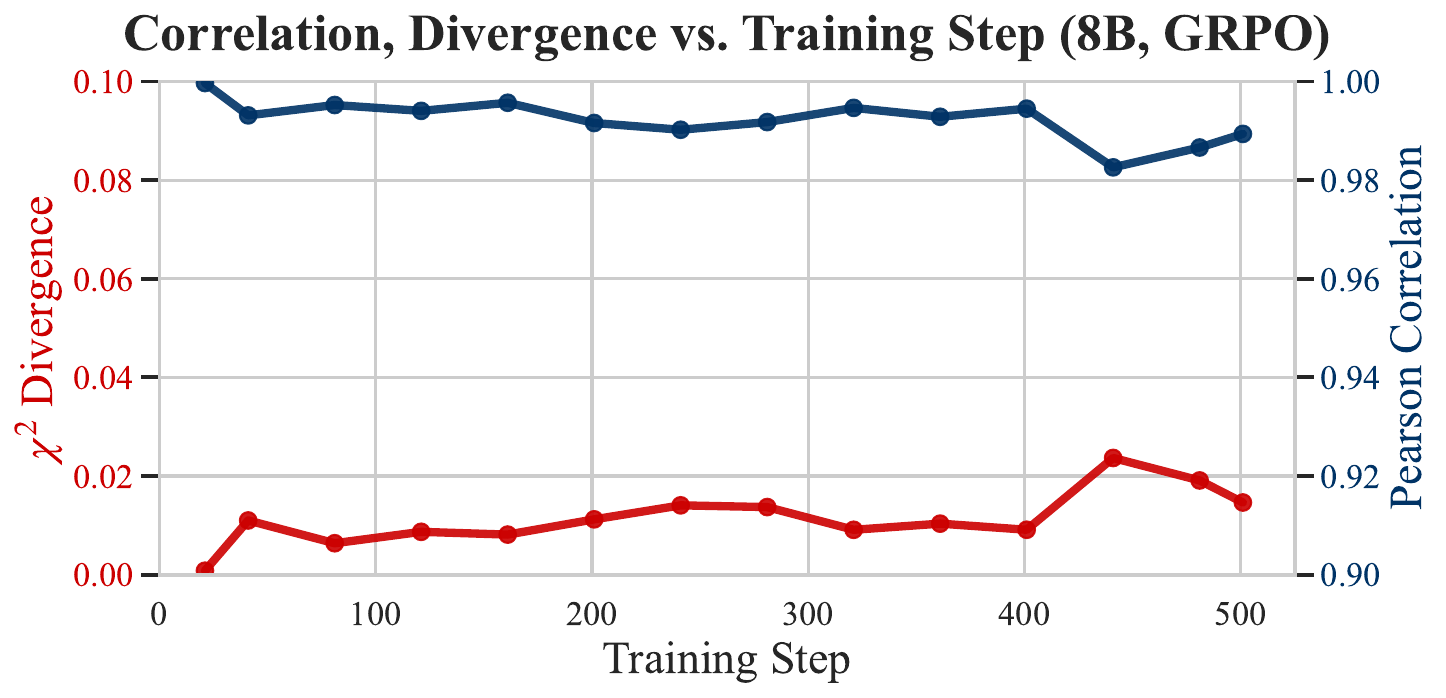}
        \label{fig:qwen3-8b_subopt}
    \end{subfigure}
    \hfill 
    \begin{subfigure}[b]{0.48\linewidth}
        \centering
        \includegraphics[width=\linewidth]{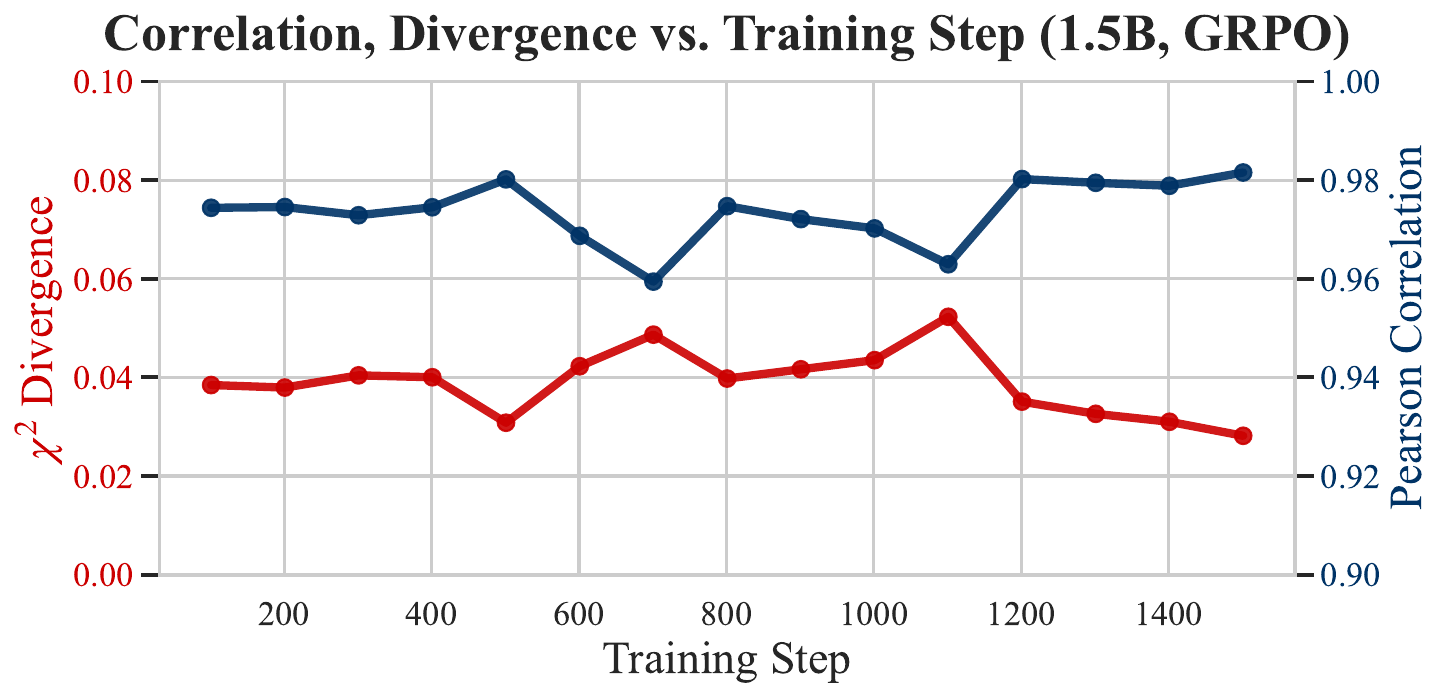}
        \label{fig:qwen1_5_grpo_subopt}
    \end{subfigure}
    \caption{\textbf{Sub-optimality metrics} for using $\abs{A_i}$ in place of $G_i$ for GRPO experiments across two model sizes. The Pearson correlations are near $1$. The cost-biased $\chi^2$-divergence between the true $p^*$ and the distribution defined by this proxy is near $0$. }
    \label{fig:subopt_main}
\end{figure}
\subsection{Datasets and Models}
We perform experiments with the Qwen2.5-Math-1.5B-Instruct \citep{yang2024qwen2}, Qwen3-4B and Qwen3-8B~\citep{yang2025qwen3} models. We use the DAPO dataset \citep{yu2025dapo} for training, designed for improving accuracy on AIME problems. We evaluate on the following benchmarks: AIME1983-2024 \citep{aime_1983_2024},  AMC \citep{aops_amc}, MATH500 \citep{hendrycks2021measuring}, and GSM8K \citep{cobbe2021training}. All experiments are performed with the Verl framework \citep{sheng2025hybridflow}, and we make only slight modifications for the purpose of our algorithm. Hyperparameters for all experiments in Appendix~\ref{app:experiment_details}.

\subsection{Cost-Aware Learning with GRPO}
\textbf{Background on GRPO. }In GRPO, training alternates between two stages. The first stage generates many responses for each prompt in order to build a dataset for the second stage.
The second stage is where policy gradient updates are made. The gradient updates are made on the current set of prompts and responses, which can be considered as the current dataset. 

We assume the training dataset has ``verifiable rewards,'' such as a math dataset where the answers are known. The reward is $1$ if the response is correct and $0$ otherwise. Given $M$ responses for a prompt, advantages are normalized: $A_i = (r_i - \frac{1}{M}\sum_{k=1}^M r_k)/(\mathrm{std}(\{r_k\}_{k=1}^M)).$ One of the key benefits of GRPO-style algorithms is that the advantage is very cheap to obtain. 

\textbf{Defining $c_i$. }
The notion of cost fits seamlessly in this setting of GRPO because different prompt and response pairs have different lengths. The amount of FLOPs required in the policy gradient step depends linearly on the total number of tokens in the prompt and response. Hence, for each prompt and response pair, we define $c_i$ to be the sum of the length of the prompt and the length of the response.

\textbf{Proxy for $G_i$. } It remains to find a strong proxy for $G_i$; 
recall that $G_i$ represents the norm of the gradient of a sample. In language model training, it would not be logical to compute the norm of the gradient of each sample in order to obtain an optimal sampling distribution, as this has a very high computational burden and defeats the purpose. Crucially, \textit{we estimate $G_i$ with $\abs{A_i}$}, the magnitude of the advantage of sample $i$. This is motivated by the expression for the policy gradient, in which the magnitude of the advantage appears to play a large role.

We justify this choice by observing low cost-biased  $\chi^2$-divergence between the sampling proxy obtained this way and the true $p^*$ distribution (invoking Theorem~\ref{thm:distribution_subopt}), as well as high Pearson correlation between $G_i$ and $\abs{A_i}$ (invoking Theorem~\ref{thm:proxy_subopt}). In Figure~\ref{fig:subopt_main}, we report the Pearson correlation between $\abs{A_i}$ and $G_i$ across training for GRPO with the 8B and 1.5B models. In order to obtain $G_i$, we compute true sequence-level gradients. Upon computing the true gradient norms, we compute the cost-biased $\chi^2$ divergence between the distributions obtained using $\abs{A_i}$ and $G_i$ respectively. We report sub-optimality metrics for all training settings in Appendix~\ref{app:suboptimality}.
\label{subsec:cal+grpo}
\paragraph{Algorithm. } Equipped with notions of cost and a gradient proxy, we combine cost-aware learning with GRPO in the policy-gradient stage of training. When we perform the policy gradient steps, we fill mini-batches according to the $p^*$ distribution rather than simply using all samples once. The $p^*$ distribution is obtained by letting $p^*_i \propto \abs{A_i}/ \sqrt{c_i}$. Thus, computing $p^*$ simply requires obtaining the advantage for all samples.  We use importance sampling to mitigate bias. Aside from this, we make no changes to the GRPO pipeline. The exact procedure is detailed in Algorithm~\ref{alg:ca_grpo}.

\begin{figure*}[t]
    \centering
    \begin{subfigure}[b]{0.49\textwidth}
        \centering
        \includegraphics[width=\textwidth]{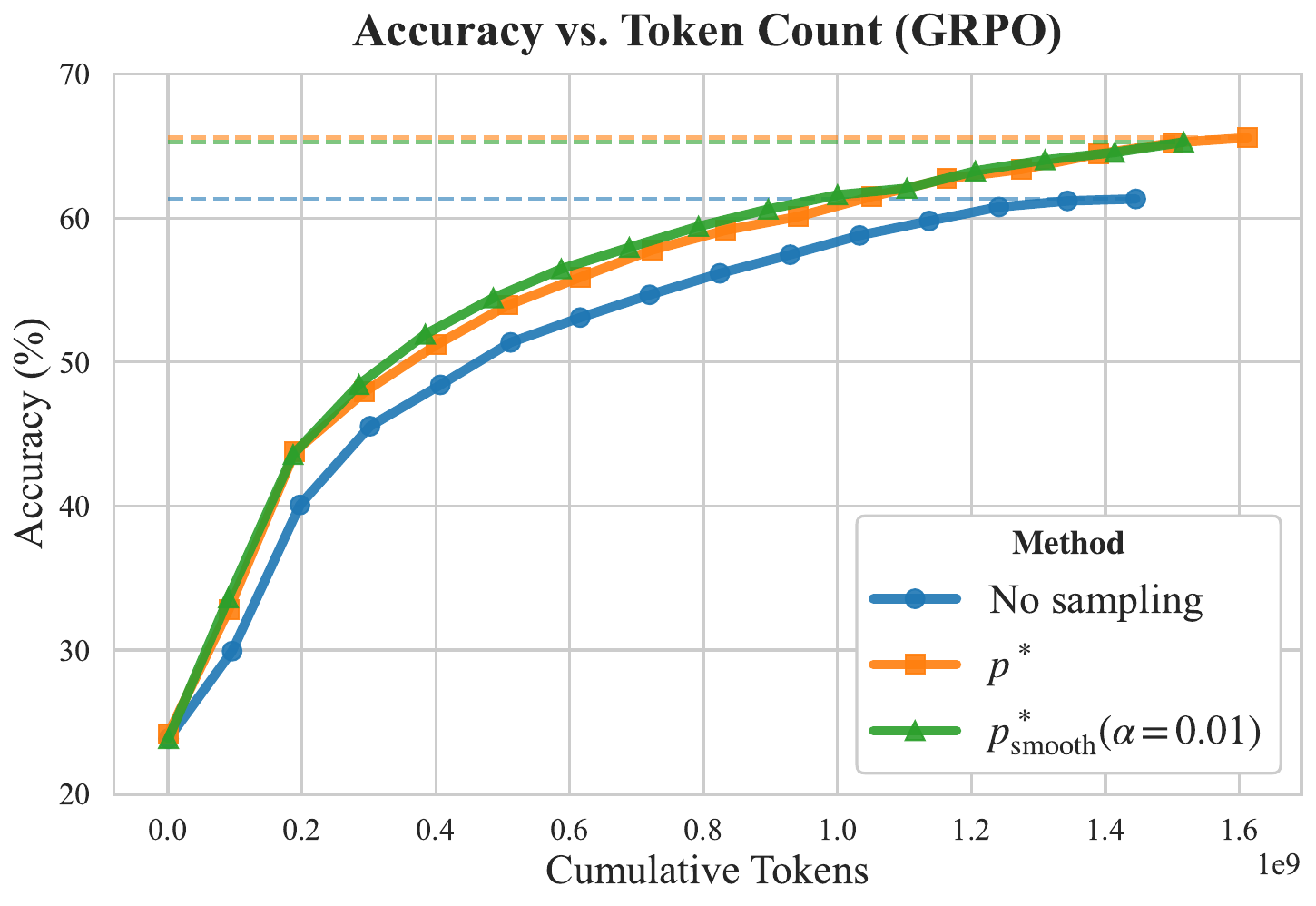}
        \label{fig:plot1}
    \end{subfigure}
    \hfill
    \begin{subfigure}[b]{0.49\textwidth}
        \centering
        \includegraphics[width=\textwidth]{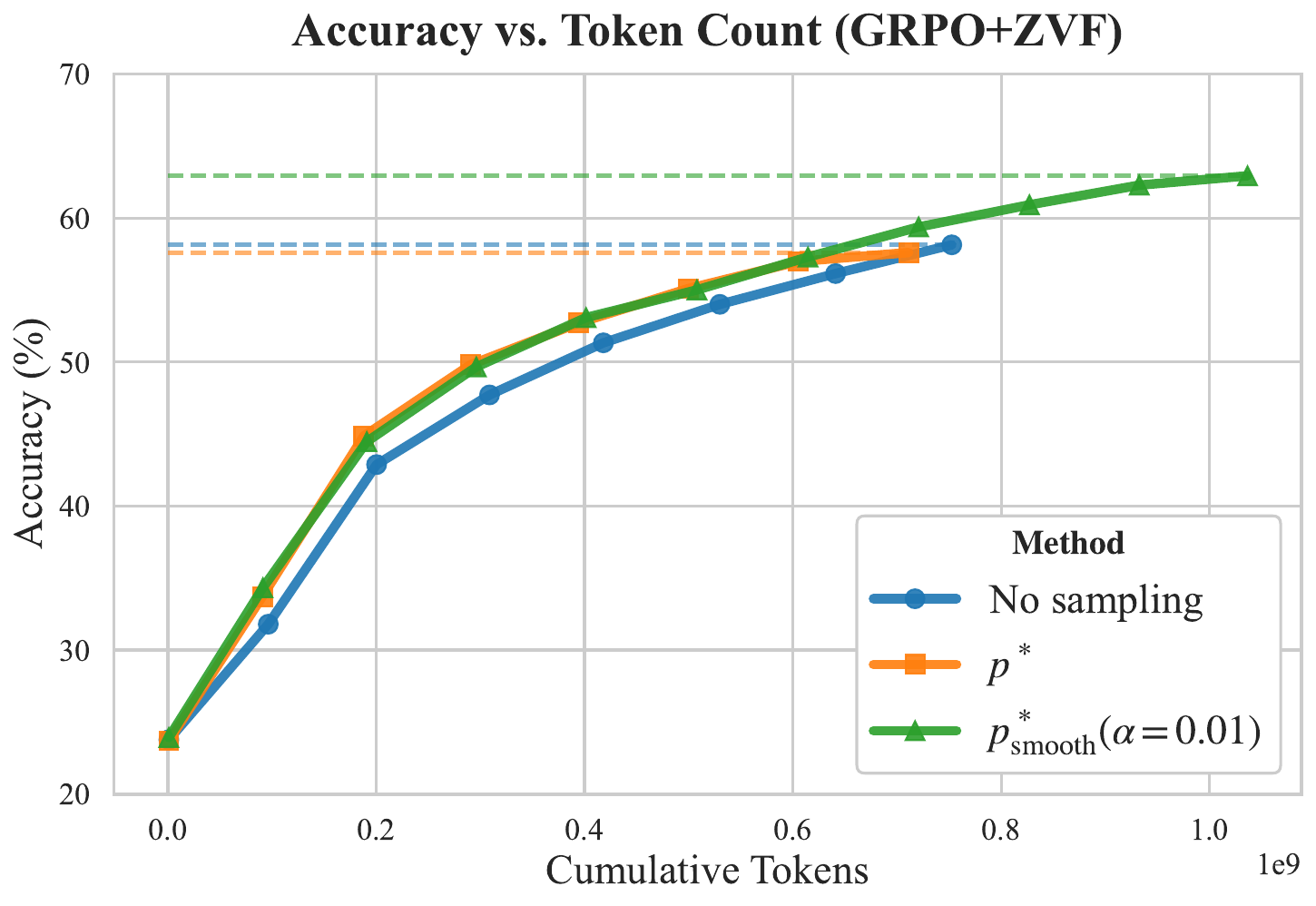}
        \label{fig:plot2}
    \end{subfigure}
    \caption{\textbf{Accuracy vs. token count} on AIME throughout training for both GRPO and GRPO+ZVF settings for the 1.5B model. We plot the accuracy on the $y$-axis and the cumulative number of tokens used in policy optimization on the $x$-axis, to compare the number of tokens used for a fixed accuracy. We evaluate every $100$ steps. }
    \label{fig:aime_results}
\end{figure*}


\subsection{Settings and Methods}

We evaluate combinations of two \emph{training settings} and three \emph{sampling methods}.

\noindent \textbf{Training settings. } \textbf{(1) GRPO: } Standard GRPO training as described in Section~\ref{subsec:cal+grpo}. \textbf{(2) GRPO + ZVF:} GRPO with zero-variance filtering (ZVF) \citep{yu2025dapo}, which removes zero-advantage prompts. To maintain a fixed batch size, rollouts are resampled until sufficient non-zero-advantage samples are obtained.

\begin{table*}[t]
\centering
\small
\begin{tabular}{lllccccc}
\toprule
Model & Setting & Method & AIME & AMC & MATH & GSM8K & Avg. Accuracy \\
\midrule
\multirow{6}{*}{Math-1.5B-IT}
& GRPO & No sampling & 61.3 & 64.1 & \textbf{73.2} & \textbf{86.2} & 71.2 \\
& GRPO & $p^*$ & \textbf{65.6} & \textbf{71.2} & \textbf{73.2} & 86.0 & \textbf{74.0} \\
& GRPO & $p^*_{\text{smooth}}(\alpha = 0.01)$ & 65.3 & 68.1 & 72.3 & 86.0 & 72.9 \\
\cmidrule{2-8}
& ZVF & No sampling & 58.1 & 64.7 & \textbf{73.2} & \textbf{86.1} & 70.6 \\
& ZVF & $p^*$ & 57.6 & 63.0 & 72.9 & 85.9 & 69.8 \\
& ZVF & $p^*_{\text{smooth}}(\alpha = 0.01)$ & \textbf{62.9} & \textbf{68.0} & \textbf{73.3} & 85.8 & \textbf{72.5} \\
 \midrule 
\multirow{3}{*}{4B-Base}
 &GRPO & No sampling & 56.4 & 72.1 & 86.5 & \textbf{94.5} & 77.4 \\
&GRPO & $p^*$ & 56.7 & \textbf{74.5} & 86.3 & 94.2 & \textbf{77.9} \\
& GRPO & $p^*_{\text{smooth}}(\alpha = 0.01)$ & \textbf{56.9} & 72.0 & \textbf{86.8} & 94.2 & 77.4 \\
\midrule
\multirow{3}{*}{8B-Base}
& GRPO & No sampling & 58.3 & 72.8 & \textbf{87.0} & \textbf{96.0} & 78.5 \\
& GRPO & $p^*$ & \textbf{58.5} & \textbf{74.7} & 86.7 & 95.2 & \textbf{78.8} \\
& GRPO & $p^*_{\text{smooth}}(\alpha = 0.01)$ & \textbf{58.5} & 74.0 & 86.7 & 94.9 & 78.5 \\
\bottomrule
\end{tabular}
\caption{\textbf{Cost-aware sampling preserves model performance.} For each setting and method, we report the best averaged checkpoint accuracy across four benchmarks.  $p^*$ sampling and variants consistently preserve overall model performance. }
\label{tab:perf_merged}
\end{table*}

\noindent \textbf{Sampling methods.} \textbf{(1) No sampling:} Standard training on all generated rollouts without resampling or importance weighting. \textbf{(2) $p^*$ sampling:} We sample rollouts according to $p^*$ (Section~\ref{subsec:cal+grpo}, Algorithm~\ref{alg:ca_grpo}) and apply importance weighting to correct for bias. \textbf{(3) Smoothed $p^*$ sampling:} We also consider $p^*_{\text{smooth}}(\alpha) = (1 - \alpha)p^* + \alpha \mathcal{U}$, where $\mathcal{U}$ is the uniform distribution.

For the 1.5B model, we evaluate all settings and sampling methods. For the 4B and 8B models, we restrict our attention to the GRPO setting because it outperforms GRPO+ZVF in the smaller scale setting. Implementation and hyperparameter details are provided in Appendix~\ref{app:experiment_details}.

\subsection{Results}
We evaluate all methods across two primary axes: benchmark accuracy and the cumulative tokens used in policy optimization. Since token count is directly proportional to FLOPs in the policy gradient step, this measures the cost-accuracy tradeoff of each strategy. All reported accuracies denote \texttt{pass@1/mean@32}, the average correctness of querying the model $32$ times. In all figures, we plot up to peak accuracy. We focus on AIME accuracy, as this is the intended downstream benchmark for the DAPO training set.

\textbf{Qwen2.5-Math-1.5B-Instruct Results. }  Figure~\ref{fig:aime_results} plots the cumulative policy optimization tokens against AIME accuracy.  Under standard GRPO, both $p^*$ and $p^*_{\text{smooth}}(\alpha = 0.01)$ match the baseline's peak accuracy using 28\% and 32\% fewer tokens respectively and ultimately surpass the baseline's maximum accuracy by 5 percentage points (pp). 

Under GRPO+ZVF, $p^*_{\text{smooth}}(\alpha = 0.01)$ maintains an advantage, achieving the baseline's peak accuracy using 13\% fewer tokens before ultimately outperforming it by 5pp. 

Results on AMC benchmark (visualized in Figure~\ref{fig:qwen_1_5_amc_results}) show a similar pattern: (1) $p^*$ sampling with GPRO requires 34\% fewer tokens than the baseline to achieve the same accuracy and improves peak accuracy by 7pp, and (2) $p^*_\text{smooth}(0.01)$ sampling with GRPO+ZVF requires 5\% fewer tokens than the baseline to achieve the same accuracy and improves peak accuracy by 3pp.



\textbf{Qwen3-4B Results. }On AIME (Figure~\ref{fig:qwen3_4b}), 35\% fewer tokens are needed for $p^*$ sampling to reach its absolute peak than no sampling requires to come within 1\% of this accuracy. On AMC (Figure~\ref{fig:qwen3-4b-amc}), $p^*$ sampling exceeds the baseline by 2pp, and requires \textasciitilde24\% fewer tokens to reach the peak accuracy. 

\textbf{Qwen3-8B Results.} Figure~\ref{fig:qwen3-8b-results} shows AIME accuracy. $p^*$ sampling achieves the baseline's peak accuracy using $51\%$ fewer tokens. Because the baseline's late-stage improvement is marginal, we also report an asymmetric comparison: matching $p^*$ when it reaches $99\%$ of its peak accuracy. Even under this relaxed threshold, $p^*$ requires $17\%$ fewer tokens than the baseline. This efficiency trend extends to the AMC benchmark (Figure~\ref{fig:qwen3-8b-amc}), where $p^*$ sampling reaches the baseline's maximum accuracy using $47\%$ fewer tokens. 
 We also explored smoothed variants (Figure~\ref{fig:qwen3-8b-aime-all}). These improve upon the baseline, but $p^*$ performs best, possibly due to the high fidelity of $\abs{A_i}$ as a proxy (Figure~\ref{fig:subopt_main}).

Lastly, Table~\ref{tab:perf_merged} reports the best checkpoint accuracy for all models and methods. Crucially, these results demonstrate that cost-aware sampling preserves model performance with respect to standard GRPO training; in fact, it is even able to improve upon baseline performance. 


Across the board, cost-aware sampling preserves, and often improves upon, the accuracy of standard GRPO, while substantially reducing the token cost of policy optimization, as seen in the results above. Across model scales, we showed that $\abs{A_i}$ is a reliable proxy for $G_i$, and that cost-aware sampling reduces the cost of policy optimization. While our primary focus is on the cost of policy optimization, results also indicate that fewer overall training steps are needed for cost-aware sampling.

\subsection{Ablations}
\begin{figure}[t]
    \centering
    
    \begin{subfigure}[b]{0.48\linewidth}
        \centering
        \includegraphics[width=\linewidth]{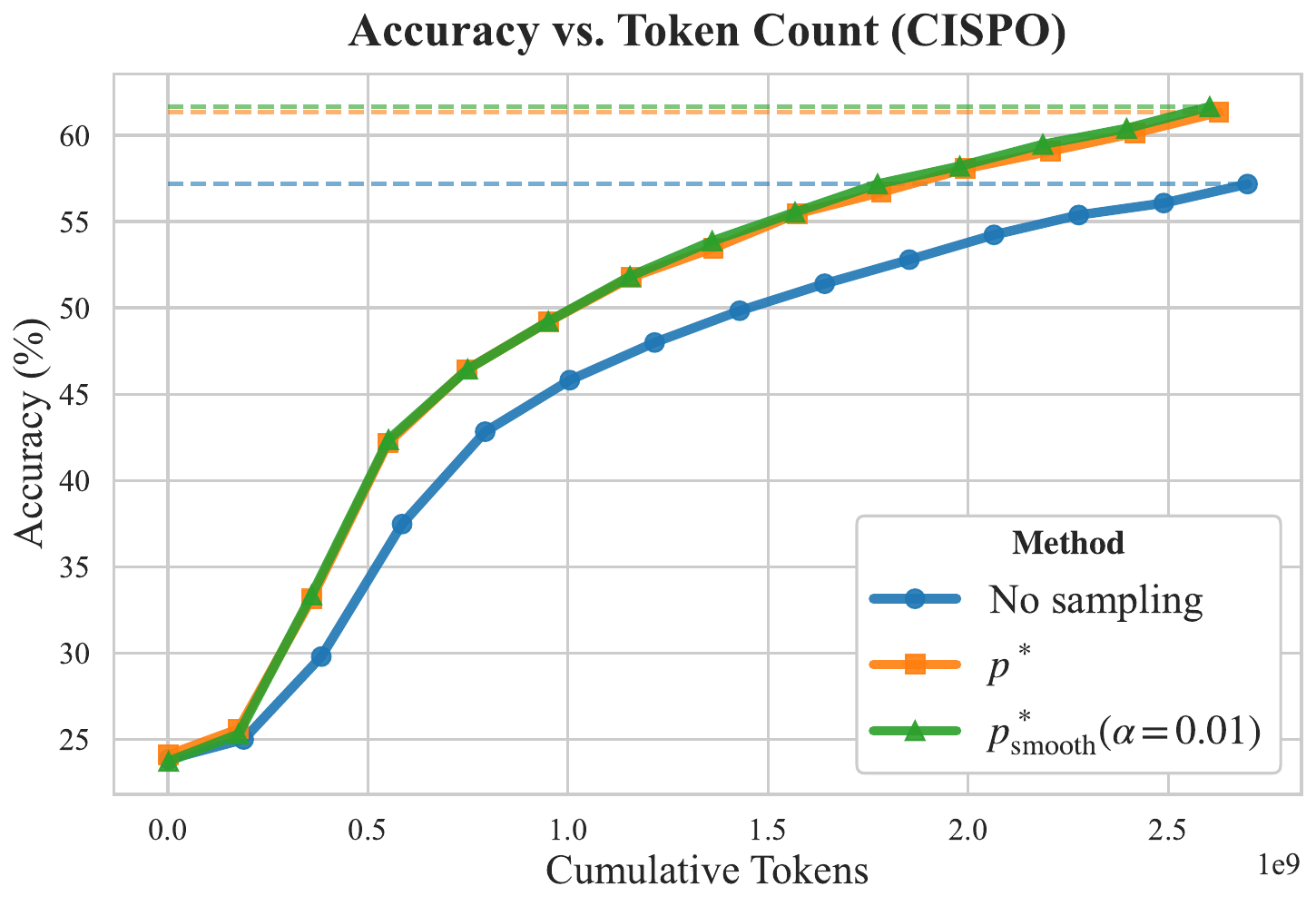}
        \caption{\textbf{CISPO. } AIME accuracy (\texttt{pass@1/mean@32}) for training Qwen2.5-Math-1.5B-Instruct with CISPO. Cost-aware learning is robust to training objectives. }
        \label{fig:scale_rl}
    \end{subfigure}
    \hfill 
    \begin{subfigure}[b]{0.48\linewidth}
        \centering
        \includegraphics[width=\linewidth]{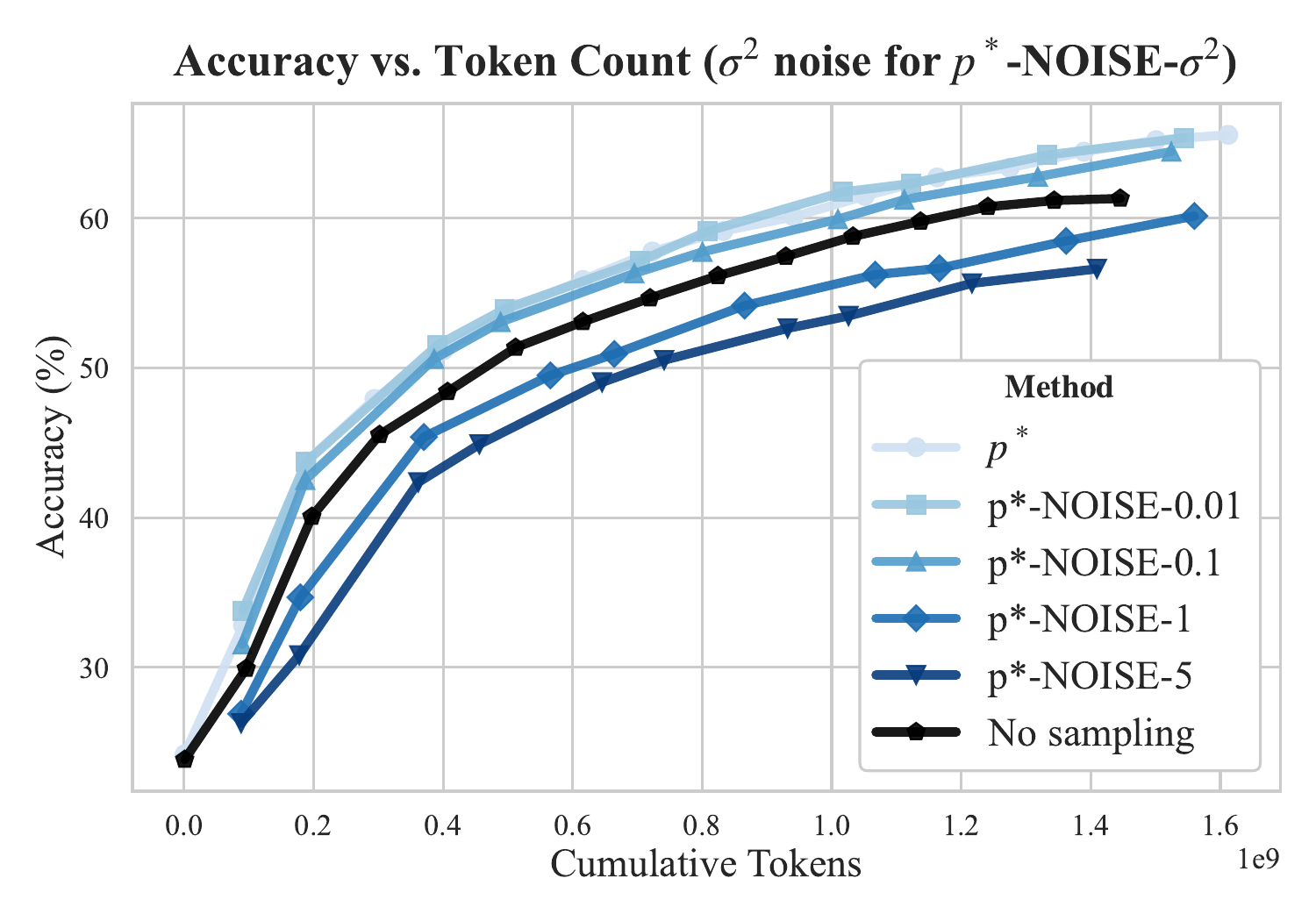}
        \caption{\textbf{Noisy proxies. } Math-1.5B with GRPO when adding zero-mean, $\sigma^2$-variance noise to the proxy $\abs{A_i}$ prior to computing the sampling distribution.}
        \label{fig:noisy_grpo}
    \end{subfigure}
    \label{fig:both_plots}
\end{figure}

\noindent \textbf{Training objective.} We consider a variant of the GRPO objective, CISPO \citet{chen2025minimax}. CISPO applies a stop-gradient to the importance weight of the objective. We base our hyperparameters on findings from~\citet{khatri2025art}. Results, reported in Figure~\ref{fig:scale_rl}, show that our method is not specific to the GRPO objective; rather, it generalizes across training objectives. $p^*$ sampling and $p^*_\text{smooth}(0.01)$ sampling require 31\% and 34\% fewer tokens respectively to attain the baseline's peak accuracy, and ultimately outperform by about 4pp. This finding is further supported by strong sub-optimality metrics, performance on additional benchmarks, and smoothed variants in Appendix~\ref{app:cispo_results}. 

\noindent \textbf{Proxy sub-optimality.} We ablate the noise level associated with our proxy $\abs{A_i}$. In particular, we add zero-mean noise with variance $\sigma^2$ to $\abs{A_i}$ prior to computing the sampling distribution. Results are shown in Figure~\ref{fig:noisy_grpo}. We observe that Cost-Aware GRPO is robust to noise up until $\sigma^2=0.1$, but as noise approaches $\sigma^2=1.0$, it deteriorates with respect to the baseline. 

\noindent \textbf{Choice of $\alpha$ and proxy.} In Appendix~\ref{app:p_star_ablation}, we ablate our choice of $\alpha$ and report results for $p^*_\text{smooth}$ with $\alpha=0.05$ and $\alpha=0.1$ as well. We also ablate our choice of approximating $G_i$ with $\abs{A_i}$ and consider using $p^*_i = 1 / \sqrt{c_i}$, which corresponds to substituting $G_i$ with $1$ for all samples. We observe that it does not perform as well, which implies a need for correctly approximating the gradient norm of a sample. In Appendix~\ref{app:8b_complete}, we report smoothed variants for the 4B and 8B model as well.

\section{Conclusion}
In this paper, we formalize the problem of cost-aware learning and propose Cost-Aware SGD. Our analysis derives the optimal sampling distribution for convex and strongly convex objectives, providing theoretical guarantees on the expected cost to attain an error of $\epsilon$. We extensively analyze the properties of this algorithm. 

We apply our theoretical insights to GRPO and introduce Cost-Aware GRPO. Our use of $\abs{A_i}$ as a proxy for $G_i$ in the optimal sampling distribution is supported by strong sub-optimality metrics across scales. Our experiments demonstrate the broad effectiveness of this approach; it significantly reduces the tokens used in policy optimization, while maintaining or exceeding the peak performance of standard baselines.

\paragraph{Future Directions.}
In Cost-Aware GRPO, we primarily rely on the magnitude of the advantage as a proxy for the gradient norm; identifying superior proxies remains an interesting direction for future research. Furthermore, having demonstrated that cost-aware sampling generalizes successfully from standard GRPO to the CISPO objective, a natural next step is to explore its integration into a broader class of RLVR and RLHF algorithms. It would also be of interest to dynamically set $\alpha$ in smoothed variants of Cost-Aware GRPO in order to obtain superior performance. Lastly, closing the gap between the upper and lower bounds in Section~\ref{sec:importance_sampling} would be of theoretical significance.

%% file: neurips/appendix/related_work.tex
\newpage
\section{Related Work}
\label{app:related_work}
We describe further related work.

\paragraph{Training Data Pruning/ Selection.} 
There is a large literature focused on pruning a training dataset. Typically, the objective is to obtain a better predictor or maintain performance with fewer samples. Related works include \citep{sorscher2022beyond,paul2021deep, xie2023data, chen2023alpagasus}. Our work differs from these as it explicitly considers a variable cost per sample. 

\paragraph{Active learning.}
There is a significant body of work studying the problem of active learning, which considers when to request more information about training samples. We refer the reader to \citep{dasgupta2011two} for an in-depth overview of the literature. Most related to our work, \citet{beygelzimer2009importance, pmlr-v119-cortes20a, cortes2019region} use importance weighting to probabilistically select samples for which to query a label. Generally, these algorithms fix a hypothesis set a priori, in which hypotheses are eliminated as more data is acquired. Our work shares the similarity of deciding when to request more information, however, in contrast, our work requests the \emph{gradient} of a sample. Furthermore, we do not use a current hypothesis set in order to make these decisions.

%% file: neurips/appendix/algorithms.tex
\newpage
\section{Algorithms}
\label{app:algorithms}

\subsection{Cost-Aware SGD}

\begin{algorithm}[htbp]
\caption{Cost-Aware SGD}
\label{alg:cost_weighted_sgd}
\begin{algorithmic}[1]
   \STATE {\bfseries Input:} Initial point $x_1$, iterations $T$, step size $\eta$, sampling distribution $\sfp^*$.
   \FOR{$t = 1$ {\bfseries to} $T$}
       \STATE Sample index $i_t \sim \sfp^*$
       \STATE Query $\nabla f_{i_t}(x_t)$, incur cost $c_{i_t}$.
       \STATE Compute importance-weighted stochastic gradient:
       \STATE \quad $\tilde{g}_t = \frac{1}{n p_{i_t}} \nabla f_{i_t}(x_t)$
       \STATE Update parameter: $x_{t+1} = \Pi_{\mathcal{X}} (x_t - \eta \tilde{g}_t)$
   \ENDFOR
   \STATE {\bfseries Output:} $\bar{x}_T = \frac{1}{T} \sum_{t=1}^T x_t$ (or suffix averaging solution for strongly convex case \citep{rakhlin2011making})
\end{algorithmic}
\end{algorithm}

\subsection{Cost-Aware GRPO}

\begin{algorithm}[htbp]
   \caption{Cost-Aware GRPO}
   \label{alg:ca_grpo}
\begin{algorithmic}
   \STATE {\bfseries Input:} Initial policy $\pi_\theta$, Dataset $\cD$, Prompt batch size $n$, Rollouts per prompt $M$, Mini-batch size $B$, Iterations $K$
   \FOR{iteration $k = 1$ {\bfseries to} $K$}
       \STATE \textcolor{darkgray}{\textit{// Phase 1: Data Collection}}
       \STATE Sample prompt batch $X = \{x_1, \dots, x_n\} \sim \cD$
       \STATE Generate $M$ outputs per prompt: $y_{i,j} \sim \pi_{\theta_{\text{old}}}(\cdot|x_i)$
       \STATE Form rollout dataset $\cD_{\text{step}}$ of size $N = n \times M$ 
       \STATE Compute advantages for all samples in $\cD_{\text{step}}$
       \STATE \textcolor{darkgray}{\textit{// Phase 2: Compute Sampling Distribution}}
       \STATE Estimate weights $s_u \leftarrow G_u / \sqrt{c_u}$ for all $u \in \cD_{\text{step}}$
        \STATE Compute probabilities $p^*_u \leftarrow s_u / \sum_{v \in \cD_{\text{step}}} s_v$
       \STATE \textcolor{darkgray}{\textit{// Phase 3: Importance Weighted Updates}}
       \STATE Set number of updates $T \leftarrow \lceil N / B \rceil$
       \FOR{update step $t = 1$ {\bfseries to} $T$}
           \STATE \textcolor{darkgray}{\textit{// Sample mini-batch and optimize policy}}
           \STATE Sample mini-batch $ \sim \text{Multinomial}(\text{support}=\cD_{\text{step}}, \text{probs}=p^*, \text{size}=B)$
           \STATE Update $\pi_\theta$ by maximizing  $\hat{\mathcal{J}}^{(p^*)}_{\mathrm{GRPO}}(\theta)$ over mini-batch
       \ENDFOR
   \ENDFOR
\end{algorithmic}
\end{algorithm}

\newpage

%% file: neurips/appendix/subset_selection.tex
\section{Subset Selection}
\label{app:subset_selection}
\begin{figure}
    \centering
\includegraphics[width=0.9 \linewidth]{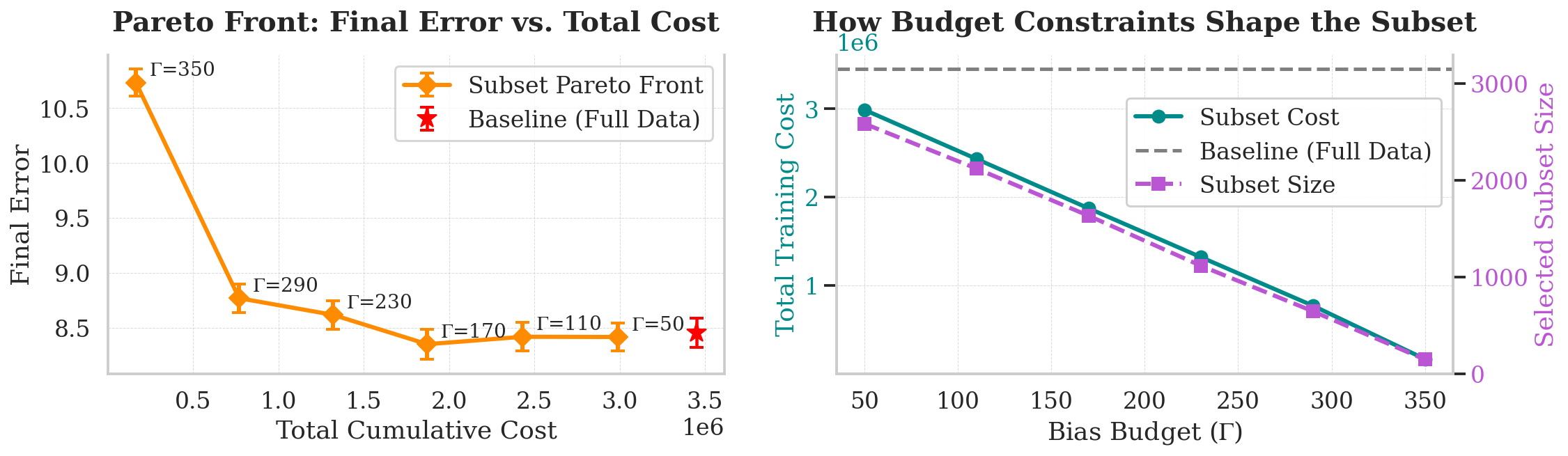}
    \caption{Synthetic validation of the greedy subset selection algorithm. $\Gamma$ is a tunable parameter representing the allowable bias budget. As it is increased, the greedy algorithm chooses fewer (and cheaper) samples to train on. }
    \label{fig:synthetic_subset}
\end{figure}
The importance sampling strategy $\sfp^*$ achieves cost $(\sum_i G_i \sqrt{c_i})^2$, but may remain expensive when high-cost samples have large gradients. To reduce cost, we consider optimizing over a subset $\pi \subseteq \{1,\dots,n\}$, defining $f_\pi(x) = \frac{1}{n}\sum_{i\in\pi} f_i(x)$. This introduces bias, as optimization converges to $x_\pi^*$ instead of $x^*$, but reduces computational cost.

\paragraph{Biased SGD on a subset.}
Let $\|\nabla f_i(x)\|\le G_i$. For a subset $\pi$ and sampling distribution $\sfq$ supported on $\pi$, define the bias $\beta_\pi = \frac{1}{n}\sum_{i\notin\pi} G_i$ and note $\|\nabla f(x)-\nabla f_\pi(x)\|\le \beta_\pi$. Standard SGD results \citep{devolder2014first} imply that with $\eta_t\propto 1/\sqrt{T}$, $\E[f(\bar{x}_T)-f(x^*)] \;\le\; \frac{D \sqrt{\sigma_\pi^2(\sfq)}}{\sqrt{T}} + D\beta_\pi,$ where $\sigma_\pi^2(\sfq)\le \frac{1}{n^2}\sum_{i\in\pi}\frac{G_i^2}{q_i}$. Thus, the cost to achieve error $\e > D\beta_\pi$ is:
$\cK_\pi(\e) = \frac{D^2}{n^2(\e - D\beta_\pi)^2} \left(\sum_{i\in\pi} G_i \sqrt{c_i}\right)^2,$
where we used the optimal sampling distribution restricted to $\pi$. 

\textbf{Optimal subset selection.}
We minimize $\cK_\pi(\e)$ subject to a bias budget $D\beta_\pi \le \Gamma$.

\begin{theorem}[Knapsack equivalence]
\label{thm:knapsack_equivalence}
This problem is equivalent to a min-cost knapsack:
\[
\min_{z_i\in\{0,1\}} \sum_i z_i G_i\sqrt{c_i}
\quad \text{s.t.} \quad
\sum_i z_i G_i \ge \sum_i G_i - \frac{n\Gamma}{D}.
\]
\end{theorem}
We defer the proof to Appendix~\ref{app:min_cost_knapsack}. A similar result for strongly convex objectives is in Appendix~\ref{app:strongly_convex_knapsack}.

\paragraph{Greedy approximation.}
A standard greedy algorithm selects items by decreasing value-to-weight density $\rho_i = v_i / \omega_i$. In our setting:
\[
\rho_i = \frac{v_i}{\omega_i} = \frac{G_i}{G_i\sqrt{c_i}} = \frac{1}{\sqrt{c_i}},
\]
so the $G_i$ terms cancel and the greedy order is determined entirely by cost $c_i$. The algorithm simply selects the cheapest components until the constraint is met, yielding a 2-approximation \citep{csirik1991heuristics}. Remarkably, the optimal greedy strategy does not depend on the Lipschitz constants $G_i$ at all. Under a similar setting to Section~\ref{subsec:synthetic_unbiased}, we implement this greedy approximation and report results in Figure~\ref{fig:synthetic_subset}. Within a smaller bias budget, subset selection is able to attain similar error at a significantly reduced cost. As the selected subset becomes much smaller, the error increases due to bias.

%% file: neurips/appendix/proofs.tex
\newpage
\section{Proofs}
\label{app:proofs}
In this section we provide the proofs and derivations for statements made in the paper.

\subsection{Variance reduction strategy derivation}
\label{app:proof_costblind}
In this case, we have:
\[
J(\sfp) = S(\sfp)C(\sfp) = \paren*{\frac{1}{n^2} \sum_{j=1}^n
\frac{G_j^2}{p_j}} \paren*{\sum_{i=1}^n p_i},
\]
where $\sfp \in \Delta_n$.
 From Cauchy-Schwarz, we can lower bound $J(p)$:
\[
J(\sfp) \geq \frac{1}{n^2} \paren*{\sum_{i = 1}^n G_i}^2.
\]
This lower bound  is attained when $G_i / \sqrt{p_i} \propto \sqrt{p_i}$, or equivalently, $p_i \propto G_i$.
Hence, when costs are uniform, we focus on minimizing the second moment. This gives:
\begin{equation}
    (p_\text{var})_i = \frac{G_i}{\sum_j G_j}.
\end{equation}

When we apply $p_\text{var}$ to the non-uniform cost setting, the total expected cost to attain error $\e$ becomes:
\begin{align*}
K_{\text{var}}(\e) &= \frac{D^2}{\e^2n^2} \paren*{\sum_{i = 1}^n G_i}^2 \paren*{\sum_{i = 1}^n (p_\text{var})_i c_i} \\
&= \frac{D^2}{\e^2n^2} \paren*{\sum_{i = 1}^n G_i}^2 \paren*{\sum_{i = 1}^n \frac{G_i}{\sum_j G_j} c_i} \\
&= \frac{D^2}{\e^2n^2} \paren*{\sum_{i = 1}^n G_i} \paren*{\sum_{i = 1}^n G_i c_i}.
\end{align*}

\subsection{Cost-Improvement \& Comparison Analysis}

\begin{theorem}[Comparison of Adaptive vs. Uniform] 
\label{app:ada_unif_comp}
Unlike the optimal strategy, the adaptive strategy $p_{\text{var}}$ is not guaranteed to outperform uniform sampling. The ratio of their costs is given by:
\begin{equation}
    \frac{K_{\text{var}}(\e)}{K_{\text{unif}}(\e)} = \frac{(\sum_{i=1}^n G_i)(\sum_{i=1}^n G_i c_i)}{(\sum_{i=1}^n G_i^2)(\sum_{i=1}^n c_i)}.
\end{equation}
Let $\E[\cdot]$ denote the expectation taken over a uniform distribution of indices $i \sim \text{Uniform}(1, \dots, n)$. Then $K_{\text{var}} \le K_{\text{unif}}$ holds if and only if:
\begin{equation}
    \E[G]\Cov(G, c) \le \E[c]\Var(G).
\end{equation}
Consequently, if gradients and costs are negatively correlated or uncorrelated, adaptive sampling is superior. However, if gradients and costs are strongly positively correlated (i.e., samples with large gradients are disproportionately expensive), uniform sampling may yield a lower total cost.
\end{theorem}

\begin{proof}
Recall that $K_{\text{unif}} \propto (\sum G_i^2)(\sum c_i)$ and $K_{\text{var}} \propto (\sum G_i)(\sum G_i c_i)$. We can rewrite these sums in terms of expectations over uniformly sampled indices. For any vector $x$, $\sum x_i = n \E[x]$. Thus:
\[
\frac{K_{\text{var}}}{K_{\text{unif}}} = \frac{(n \E[G]) (n \E[Gc])}{(n \E[G^2]) (n \E[c])} = \frac{\E[G] \E[Gc]}{\E[G^2] \E[c]}.
\]
We seek the condition under which this ratio is $\le 1$, which implies $\E[G] \E[Gc] \le \E[G^2] \E[c]$. We apply the definitions $\E[Gc] = \Cov(G, c) + \E[G]\E[c]$ and $\E[G^2] = \Var(G) + \E[G]^2$. Substituting these into the inequality:
\[
\E[G] \paren*{\Cov(G, c) + \E[G]\E[c]} \le \paren*{\Var(G) + \E[G]^2} \E[c].
\]
Expanding both sides:
\[
\E[G]\Cov(G, c) + \E[G]^2\E[c] \le \E[c]\Var(G) + \E[G]^2\E[c].
\]
Subtracting the common term $\E[G]^2\E[c]$ from both sides yields the final condition:
\[
\E[G]\Cov(G, c) \le \E[c]\Var(G).
\]
Since $\E[G]$, $\E[c]$, and $\Var(G)$ are non-negative, this inequality always holds if $\Cov(G, c) \le 0$ (negative or zero correlation). It is violated only if $\Cov(G, c)$ is sufficiently large and positive.
\end{proof}

\subsection{Min Cost Knapsack}
\begin{theorem}[Equivalence to Min-Knapsack Problem]
\label{app:min_cost_knapsack}
Let $\Gamma$ be the maximum allowable bias budget such that convergence is possible.  Finding the optimal subset $\pi^*$ that minimizes the expected total cost subject to this budget is equivalent to solving the \textbf{Min-Cost Knapsack Problem}.

Specifically, let $z_i \in \{0, 1\}$ be the indicator variable where $z_i=1$ if $i \in \pi$. The problem is:
\begin{equation}
\label{eq:knapsack}
    \min_{z} \sum_{i=1}^n z_i \omega_i \quad \text{subject to} \quad \sum_{i=1}^n z_i v_i \geq V_{\text{req}},
\end{equation}
where the ``item cost'' is $\omega_i = G_i \sqrt{c_i}$, the ``item value'' is $v_i = G_i$, and the required value coverage is $V_{\text{req}} = \sum_{j=1}^n G_j - \frac{n \Gamma}{D}$.
\end{theorem}
\begin{proof}
    We aim to minimize the cost complexity factor derived for the optimal distribution restricted to $\pi$: $J(\pi) = (\sum_{i \in \pi} G_i \sqrt{c_i})^2$.
    The constraint is on the bias error contribution $D \beta_\pi \le \Gamma < \e$. Recalling the definition of bias:
    \[
    \beta_\pi \le \frac{1}{n} \sum_{j \notin \pi} G_j \implies D \paren*{\frac{1}{n} \sum_{j \notin \pi} G_j} \le \Gamma.
    \]
    Minimizing the sum of gradients \textit{excluded} from $\pi$ is equivalent to maximizing the sum of gradients \textit{included} in $\pi$. We rewrite the constraint:
    \[
    \sum_{j \notin \pi} G_j \le \frac{n \Gamma}{D} \implies \sum_{j=1}^n G_j - \sum_{i \in \pi} G_i \le \frac{n \Gamma}{D} \implies \sum_{i \in \pi} G_i \ge \sum_{j=1}^n G_j - \frac{n \Gamma}{D}.
    \]
    Therefore, if we let $V_\text{req} = \sum_{j=1}^n G_j - \frac{n \Gamma}{D}$, let $\omega_i = G_i \sqrt{c_i}$ as the penalty for including item $i$, and let $v_i = G_i$ as the gain toward satisfying the constraint, we recover the formulation in Equation~(\ref{eq:knapsack}).
\end{proof}
\noindent Using this algorithm, the worst case $\E[K(\e)]$, where $\e > \Gamma$, for $D\beta_\pi \leq \Gamma$ is:
\[
\min_{\{\pi: \beta_\pi \leq \Gamma / D\}}  \frac{D^2 J(\pi)}{n^2(\e - \Gamma)^2}.
\]

The bias constraint $D\beta_\pi \le \Gamma$ determines the required gradient coverage 
$V_{\mathrm{req}} = \sum_i G_i - n\Gamma/D$, and thus defines a specific Min-Cost Knapsack instance as in Theorem~\ref{thm:knapsack_equivalence}. By varying $\Gamma$ over its feasible range and solving the associated knapsack problem for each choice, we obtain an optimal subset $\pi(\Gamma)$ minimizing the total cost subject to that bias budget. Since the total expected complexity 
\[
\E[K_\pi(\e)]
    = \frac{D^2 J(\pi)}{n^2(\e - D\beta_\pi)^2}
\]
depends on both the numerator (the sampling cost) and the denominator (the bias floor), enumerating over $\Gamma$ allows us to evaluate the achievable error--cost tradeoff for every subset. Consequently, we can compute the optimal target accuracy $\e$ by selecting the bias level $\Gamma$ that yields the smallest overall expected cost.

\subsection{Strongly Convex Subset Selection}
\label{app:strongly_convex_knapsack}

\begin{theorem}[Subset Selection for Strongly Convex Functions]
    Assume $f$ is $\mu$-strongly convex. Let $z_i \in \{0, 1\}$ be the indicator variable for including data point $i$ in subset $\pi$. We seek to minimize the cost complexity $\sum z_i G_i \sqrt{c_i}$ subject to a budget $\Gamma$ on the function value suboptimality $f(x_\pi^*) - f(x^*) \le \Gamma$.
    
    This problem is structurally equivalent to the Min-Cost Knapsack formulation in Theorem~\ref{thm:knapsack_equivalence}, with a modified requirement $V_{\text{req}}$. The optimal solution to the LP relaxation is obtained by selecting indices $i$  strictly based on increasing  cost $c_i$.
\end{theorem}

\begin{proof}
    Using the Polyak-Lojasiewicz (PL) inequality and the fact that $\nabla f_S(x_S^*) = 0$, the gradient at the proxy minimizer is determined solely by the excluded gradients:
    \[
        \nabla f(x_\pi^*) = \frac{1}{n} \sum_{j \notin \pi} \nabla f_j(x_\pi^*).
    \]
    We bound the norm of this gradient using the triangle inequality ($\|\sum v\| \le \sum \|v\|$) and the gradient upper bound definition ($\|\nabla f_j(x)\| \le G_j$):
    \begin{align*}
        \|\nabla f(x_\pi^*)\| &= \norm*{ \frac{1}{n} \sum_{j \notin \pi} \nabla f_j(x_\pi^*) } \\
        &\le \frac{1}{n} \sum_{j \notin \pi} \norm{ \nabla f_j(x_\pi^*) } \\
        &\le \frac{1}{n} \sum_{j \notin \pi} G_j.
    \end{align*}
    Substituting this back into the PL inequality ($f(x_\pi^*) - f(x^*) \le \frac{1}{2\mu} \|\nabla f(x_\pi^*)\|^2$) yields:
    \[
        f(x_\pi^*) - f(x^*) \le \frac{1}{2\mu n^2} \paren*{ \sum_{j \notin \pi} G_j }^2 = \frac{1}{2\mu n^2} \paren*{ \sum_{i=1}^n (1-z_i) G_i }^2.
    \]
    We enforce the constraint $f(x_\pi^*) - f(x^*) \le \Gamma$. Because the gradient norms $G_i$ are non-negative, we can take the square root of the inequality to linearize the constraint without altering the optimization region:
    \[
        \sum_{i=1}^n (1-z_i) G_i \le n\sqrt{2\mu\Gamma}.
    \]
    This is now a linear constraint on the excluded mass. Let $B = n\sqrt{2\mu\Gamma}$ be the effective budget for the excluded gradients. We rewrite this in terms of the included variables $z_i$:
    \[
        \sum_{j=1}^n G_j - \sum_{i=1}^n z_i G_i \le B \implies \sum_{i=1}^n z_i G_i \ge \sum_{j=1}^n G_j - B.
    \]
    The optimization problem is to minimize $\sum z_i (G_i \sqrt{c_i})$ subject to $\sum z_i G_i \ge V_{\text{req}}$. This is identical to Equation~\ref{eq:knapsack}. 
    This problem is structurally equivalent to the Min-Cost Knapsack formulation in Theorem~\ref{thm:knapsack_equivalence}, with a modified requirement $V_{\text{req}}$. Its LP relaxation is a fractional knapsack problem whose optimal solution is obtained by ordering indices by increasing evaluation cost $c_i$ (equivalently, decreasing density $v_i/\omega_i$).
\end{proof}

\subsection{Sub-Optimality Proofs}
\label{app:proofs_subopt}
\begin{theorem}[Sub-optimality Gap via Cost-Biased Divergence] Let $p^*$ be the optimal cost-aware sampling distribution  and $p'$ be any empirical sampling distribution. Let the expected cost per iteration for a distribution $p$ be $C(p) = \sum_{i=1}^n p_i c_i$.We define the cost-biased distribution $\tilde{p}$ corresponding to $p$ as:
$$\tilde{p}_i = \frac{p_i c_i}{C(p)}.$$
The optimality gap is exactly determined by the Pearson $\chi^2$-divergence between the cost-biased optimal distribution and the cost-biased empirical distribution:
$$\frac{J(p')}{J(p^*)} = 1 + D_{\chi^2}(\tilde{p}^* || \tilde{p}').$$
\end{theorem}

\begin{proof*} 

Recall that minimizing the total expected cost is equivalent to minimizing the product $J(p) = S(p)C(p)$. The optimal distribution is defined as $p^*_i = \frac{G_i / \sqrt{c_i}}{Z_{den}}$, where $Z_{den} = \sum_{j=1}^n \frac{G_j}{\sqrt{c_j}}$. We can rearrange this to express the true gradient norms in terms of $p^*_i$:
$$G_i = Z_{den} \, p^*_i \sqrt{c_i}.$$
Substituting this expression for $G_i$ into the variance term $S(p') = \frac{1}{n^2} \sum_{i=1}^n \frac{G_i^2}{p'_i}$  yields the objective for the empirical distribution:
$$J(p') = \frac{Z_{den}^2}{n^2} C(p') \left( \sum_{i=1}^n \frac{(p^*_i)^2 c_i}{p'_i} \right).$$
Evaluating this same expression for the optimal distribution $p^*$ (where $p'_i = p^*_i$) gives $J(p^*) = \frac{Z_{den}^2}{n^2} C(p^*)^2$. Taking the ratio of the two costs isolates the gap:
$$\frac{J(p')}{J(p^*)} = \frac{ C(p') \sum_{i=1}^n c_i \frac{(p^*_i)^2}{p'_i} }{ C(p^*)^2 }$$

By the definition of the cost-biased distribution, we have $p_i = \tilde{p}_i \frac{C(p)}{c_i}$. Substituting this expression for both $p^*$ and $p'$ into the summation term:
$$\sum_{i=1}^n c_i \frac{(p^*_i)^2}{p'_i} = \sum_{i=1}^n c_i \frac{\left( \tilde{p}^*_i \frac{C(p^*)}{c_i} \right)^2}{\tilde{p}'_i \frac{C(p')}{c_i}}.$$
Simplifying the terms inside the sum:$$= \sum_{i=1}^n c_i \frac{ (\tilde{p}^*_i)^2 \frac{C(p^*)^2}{c_i^2} }{ \tilde{p}'_i \frac{C(p')}{c_i} } = \frac{C(p^*)^2}{C(p')} \sum_{i=1}^n \frac{(\tilde{p}^*_i)^2}{\tilde{p}'_i}.$$

 Substituting this simplified summation back into our original cost ratio:
$$\frac{J(p')}{J(p^*)} = \frac{C(p')}{C(p^*)^2} \left( \frac{C(p^*)^2}{C(p')} \sum_{i=1}^n \frac{(\tilde{p}^*_i)^2}{\tilde{p}'_i} \right)=\sum_{i=1}^n \frac{(\tilde{p}^*_i)^2}{\tilde{p}'_i}.$$
Using the definition of the Pearson $\chi^2$-divergence, $D_{\chi^2}(P || Q)=\sum \frac{P^2}{Q} -1$, we arrive at the identity.
\end{proof*}

\begin{theorem}
\label{thm:proxy_subopt}
Let $p^*$ be the optimal cost-aware sampling distribution using the true gradient norms $G_i$, and let $p'$ be the empirical sampling distribution using estimated gradient norms $G'_i$, such that $p'_i \propto G'_i / \sqrt{c_i}$.Assume the estimates follow an additive noise model $G'_i = G_i + \epsilon_i$, where $\epsilon_i$ are independent, zero-mean noise terms with variance $\sigma_\epsilon^2$, independent of $c_i$. Let $\rho$ be the Pearson correlation coefficient between the true norms $G_i$ and the estimates $G'_i$. Assuming the noise $\epsilon_i$ is relatively small compared to $G_i$, the expected optimality gap is approximately bounded by:$$\frac{\mathbb{E}[J(p')]}{J(p^*)} \approx 1 + \left( \frac{1 - \rho^2}{\rho^2} \right) \sigma_G^2 \frac{\sum_{i=1}^n \frac{\sqrt{c_i}}{G_i}}{\sum_{i=1}^n G_i \sqrt{c_i}}$$where $\sigma_G^2$ is the variance of the true gradient norms.
\end{theorem}

\begin{proof*} The cost objective for a distribution $p$ is $J(p) = S(p)C(p)$. For the estimated distribution $p'$, we have $p'_i = \frac{G'_i / \sqrt{c_i}}{\sum_{j=1}^n G'_j / \sqrt{c_j}}$. Substituting this into the cost objective yields:
$$J(p') = \frac{1}{n^2} \left( \sum_{i=1}^n \frac{G_i^2 \sqrt{c_i}}{G'_i} \right) \left( \sum_{i=1}^n G'_i \sqrt{c_i} \right).$$

Under the model $G'_i = G_i + \epsilon_i$, the covariance between $G$ and $G'$ is strictly the variance of $G$, namely $Cov(G, G') = \sigma_G^2$. The variance of $G'$ is $Var(G') = \sigma_G^2 + \sigma_\epsilon^2$. The Pearson correlation coefficient $\rho$ is defined as:
$$\rho = \frac{Cov(G, G')}{\sigma_G \sigma_{G'}} = \frac{\sigma_G^2}{\sigma_G \sqrt{\sigma_G^2 + \sigma_\epsilon^2}} = \frac{\sigma_G}{\sqrt{\sigma_G^2 + \sigma_\epsilon^2}}$$
Solving for the noise variance $\sigma_\epsilon^2$ in terms of $\rho$:
$$\sigma_\epsilon^2 = \sigma_G^2 \left( \frac{1 - \rho^2}{\rho^2} \right).$$

Because $G'_i$ appears in the denominator of the second moment, we apply a second-order Taylor expansion of $1/G'_i$ centered around $G_i$:$$\frac{1}{G_i + \epsilon_i} \approx \frac{1}{G_i} - \frac{\epsilon_i}{G_i^2} + \frac{\epsilon_i^2}{G_i^3}$$Substituting this into the first term of $J(p')$ and taking the expectation with respect to the noise $\epsilon$:

$$\mathbb{E}\left[ \frac{G_i^2 \sqrt{c_i}}{G'_i} \right] \approx G_i^2 \sqrt{c_i} \left( \frac{1}{G_i} - \frac{\mathbb{E}[\epsilon_i]}{G_i^2} + \frac{\mathbb{E}[\epsilon_i^2]}{G_i^3} \right) = G_i \sqrt{c_i} + \sigma_\epsilon^2 \frac{\sqrt{c_i}}{G_i}.$$

Taking the expectation of the full cost $J(p')$, noting that the two sums are weakly dependent for large $n$ and $\mathbb{E}[G'_i] = G_i$:$$\mathbb{E}[J(p')] \approx \frac{1}{n^2} \left( \sum_{i=1}^n \left( G_i \sqrt{c_i} + \sigma_\epsilon^2 \frac{\sqrt{c_i}}{G_i} \right) \right) \left( \sum_{i=1}^n G_i \sqrt{c_i} \right)$$$$\mathbb{E}[J(p')] \approx \frac{1}{n^2} \left( \sum_{i=1}^n G_i \sqrt{c_i} \right)^2 + \frac{\sigma_\epsilon^2}{n^2} \left( \sum_{i=1}^n \frac{\sqrt{c_i}}{G_i} \right) \left( \sum_{i=1}^n G_i \sqrt{c_i} \right)$$Recognizing that $J(p^*) = \frac{1}{n^2} \left( \sum_{i=1}^n G_i \sqrt{c_i} \right)^2$, we divide by $J(p^*)$:$$\frac{\mathbb{E}[J(p')]}{J(p^*)} \approx 1 + \sigma_\epsilon^2 \frac{\sum_{i=1}^n \frac{\sqrt{c_i}}{G_i}}{\sum_{i=1}^n G_i \sqrt{c_i}}.$$
Finally, substituting $\sigma_\epsilon^2 = \sigma_G^2 \left( \frac{1 - \rho^2}{\rho^2} \right)$ completes the proof.
\end{proof*}

%% file: neurips/appendix/lower_bound.tex
\newpage
\section{Lower bound}
\label{app:lb}

We consider the minimization of a finite-sum objective:
\begin{align*}
    \min_{x \in \mathcal{X}} F(x) = \frac{1}{n} \sum_{i=1}^n f_i(x).
\end{align*}
where each $f_i: \mathcal{X} \to \mathbb{R}$ is convex and $G$-Lipschitz continuous. We assume access to an exact first-order oracle: when queried at index $i$ and point $x$, the oracle returns the exact gradient $\nabla f_i(x)$.
Crucially, we assume there is a variable cost $c_i > 0$ to querying the $i$-th component (at any point $x$). 

The goal is to minimize the total expected accumulated cost required to output a point $\hat{x}$ such that $\mathbb{E}[F(\hat{x}) - F(x^*)] \le \epsilon$, where the expectation is taken over any randomization in the algorithm producing $\hat x$.

\begin{theorem} 
Fix $\epsilon, G > 0$. For any number of components $n \ge G^2/\epsilon^2$, any set of nonnegative query costs $\{c_i\}_{i=1}^n$, and any (possibly randomized) algorithm with an expected error guarantee $\mathbb{E}[F(\hat{x}) - \min F(x)] \le \epsilon$, there exists a convex, $G$-Lipschitz finite-sum problem instance over $\mathcal{X} = [-1,1]$ such that the algorithm's expected total query cost is at least:
\begin{align*}
    \Omega\left( \frac{G^2}{\epsilon^2} \left( \frac{1}{n} \sum_{i \in S^*} \sqrt{c_i} \right)^2 \right),
\end{align*}
where $S^* \subseteq [n]$ is any set of components such that following cost-uniformity condition holds:
\begin{align} 
 \label{eq:lb-cost-uniform}
    \frac{\max_{i \in S^*} \sqrt{c_i}}{\tfrac1n \sum_{j \in S^*} \sqrt{c_j}} \le \frac{G}{40\epsilon}.
\end{align}
\end{theorem}

\paragraph{Remark.} If we let the costs be sorted $c_1 \le c_2 \le \ldots \le c_n$, then the set $S^*$ that maximizes the lower bound (if it is nonempty) must be of the form $S = [k]$ for some $1 \leq k \leq n$, namely, the set of $k$ cheapest costs. Indeed, if $j \in S^*$ then also $i \in S^*$ for any $i < j$, since otherwise including $i$ into $S^*$ increases the lower bound, while still satisfying \cref{eq:lb-cost-uniform}: it can only increase the denominator on the LHS, while not affecting the numerator. 

For reasonably uniform costs, say perfectly uniform up to a small (constant) fraction of arbitrarily large outliers, the condition in \cref{eq:lb-cost-uniform} is satisfied by fixing $S^*$ to contain all components except for the outliers, as long as $G/\epsilon$ is larger than a constant.
It could be, however, that the maximal $S^*$ is empty (and the lower bound is zero) if the costs are highly non-uniform. For example, if the costs are exponentially increasing, say $c_i = 4^i$ for all $i$, then the ratio between the maximum and the sum of $\sqrt{c_i}$ over any set of the form $S^* = [k]$ is between $\tfrac12$ and $1$, and so the cost-uniformity condition implies that $\epsilon \leq \frac{G}{20 n}$. Thus, for any larger value of $\epsilon$ there is no set $S^*$ for which the condition is satisfied.

\begin{proof}[Proof of \cref{thm:lb}]
The proof proceeds by constructing a hard distribution of instances based on linear functions and applying information-theoretic lower bounds. We break the argument into four steps.

\paragraph{Step 1: Definition of hard instance.}
Consider the 1-dimensional domain $\mathcal{X} = [-1, 1]$. We define the component functions as linear functions:
\begin{align*}
    \forall ~ i \in [n] :
    \qquad
    f_i(x) = -b_i x,
\end{align*}
where $b_i \in \mathbb{R}$ is a slope parameter. 
The global objective is $F(x) = -x \cdot \bar b$, where $\bar b = \tfrac1n \sum_{i=1}^n b_i$. The minimizer $x^*$ depends entirely on the sign of the average slope $\bar b$; namely, $x^* = \mathrm{sign}(\bar b)$.
The value of the minimum is $F(x^*) = -|\bar{b}|$, and for any candidate $\hat{x} \in [-1, 1]$, the optimality gap is
$
    F(\hat{x}) - F(x^*) = \bar{b}(x^* - \hat{x}).
$
In particular, for any algorithm producing $\hat x$,
\begin{equation} \label{eq:lb-exp-err}
    \mathbb{E}[ F(\hat{x}) - F(x^*) ] 
    \geq 
    \bar{b} \cdot \mathbb{P}[\mathrm{sign}(\hat{x}) \neq \mathrm{sign}(\bar b)]
    .
\end{equation}
%
We define the hard distribution $\mathcal{D}$ as the uniform mixture of two priors, $\mathcal{D}_0$ and $\mathcal{D}_1$, parameterized by $\Delta_1,\ldots,\Delta_n \geq 0$:
\begin{itemize}
    \item Under $\mathcal{D}_0$: $b_i = \text{clip}(z_i)$, where $z_i \sim \mathcal{N}(+\Delta_i, \sigma^2)$ are i.i.d., $\sigma = G/2$;
    \item Under $\mathcal{D}_1$: $b_i = \text{clip}(z_i)$, where $z_i \sim \mathcal{N}(-\Delta_i, \sigma^2)$ are i.i.d., $\sigma = G/2$.
\end{itemize}
Here, $\text{clip}(z_i) = \max(-G, \min(G, z))$ is a truncation to $[-G,G]$ that strictly enforces the Lipschitz condition (deterministically). The $b_i$ are then distributed according to a truncated Gaussian distribution with variance $\sigma^2$ strictly smaller than $G^2$. 
For reasons that will become apparent in later steps, we will set the parameters $\Delta_i$ so that the following conditions are met:
\begin{align} \label{eq:lb-delta-conditions}
    \frac{1}{n} \sum_{i=1}^n \Delta_i = 20\epsilon,
    \qquad\text{and}\qquad
    \Delta_i \le \frac{G}{2} \quad \text{for all $i \in [n]$}.
\end{align} 

\paragraph{Step 2: Reduction to hypothesis testing.}

Define the hypothesis testing error probability as 
\begin{align*}
    \text{$p_{err} = \tfrac12 p_{err}^0 + \tfrac12 p_{err}^1$, \; where \; $p_{err}^0 = \mathbb{P}_{\mathcal{D}_0}(\hat{x} < 0)$ \; and \; $p_{err}^1 = \mathbb{P}_{\mathcal{D}_1}(\hat{x} > 0$).}
\end{align*}
We will next show that the expected error (\cref{eq:lb-exp-err}) is lower bounded in terms of $p_{err}$.

Conditioning on $\mathcal{D}_0$, a sufficient condition for the algorithm to fail is if $\hat{x} \le 0$ and $\bar b > 0$, in which case the optimization error is at least $|\bar{b}|$; thus
\begin{align} \label{eq:lb-exp-err-D0}
    \mathbb{E}_{\mathcal{D}_0}[F(\hat{x}) - F(x^*) ] 
    \ge 
    \mathbb{E}_{\mathcal{D}_0}[ |\bar{b}| \cdot \mathbb{I}(\hat{x} \le 0) \cdot \mathbb{I}(\bar{b} > 0) ]
    .
\end{align}
We next bound the expected value (under $\mathcal{D}_0$) of the mean of truncate variables.
Let $g(\delta) = \mathbb{E}[\text{clip}(Z+\delta, -G, G)]$ where $Z \sim \mathcal{N}(0, \sigma^2)$. The mean of the clipped variable is $\mathbb{E}[b_i] = g(\Delta_i)$. By the mean value theorem, $g(\Delta_i) = g(0) + g'(\xi)\Delta_i = g'(\xi)\Delta_i$ for some $\xi \in (0, \Delta_i)$. The derivative is $g'(t) = \mathbb{P}(-G \le Z+t \le G)$. With $\sigma = G/2$ and $\Delta_i \le G/2$, the interval $[-G-\Delta_i, G-\Delta_i]$ always contains $[-G/2, G/2] = [-\sigma, \sigma]$. The Gaussian mass in this interval is $> 0.5$, thus $\mathbb{E}[b_i] \ge 0.5 \Delta_i$.
The aggregate mean then satisfies $\mathbb{E}[\bar{b}] \ge 0.5 (\frac{1}{n} \sum_{i=1}^n \Delta_i) = 10\epsilon$.

The truncated variables $b_i$ are independent and bounded in $[-G, G]$. By Hoeffding's inequality, $\bar{b}$ concentrates around its expectation: defining the event $A = \{ \bar{b} \ge 5\epsilon \}$, the probability of the complement $\delta' = \mathbb{P}_{\mathcal{D}_0}(A^c)$ is bounded by
\begin{align*}
    \delta' 
    \le \mathbb{P}_{\mathcal{D}_0}\left(\bar{b} 
    \le \mathbb{E}_{\mathcal{D}_0}[\bar{b}] - 5\epsilon\right) 
    \le \exp\left(-\frac{2 n^2 (5\epsilon)^2}{n (2L)^2}\right) 
    \le \exp\left(-\frac{3 n \epsilon^2}{G^2}\right)
    \leq 0.05
    ,
\end{align*}
where the final inequality is true since we assume $n \ge G^2/\epsilon^2$.
Conditioning on the event $A$ (where $\bar{b} \ge 5\epsilon > 0$) and recalling \cref{eq:lb-exp-err-D0}, we have
\begin{align*}
    \mathbb{E}_{\mathcal{D}_0}[F(\hat{x}) - F(x^*) ] 
    &\ge \mathbb{E}_{\mathcal{D}_0}[ |\bar{b}| \cdot \mathbb{I}(\hat{x} \le 0) \cdot \mathbb{I}(A) ] \\
    &\ge 5\epsilon \, \mathbb{P}_{\mathcal{D}_0}(\{\hat{x} \le 0\} \cap A) \\
    &\ge 5\epsilon (\mathbb{P}_{\mathcal{D}_0}(\hat{x} \le 0) - \mathbb{P}_{\mathcal{D}_0}(A^c)) \\
    &= 5\epsilon (p_{err}^0 - \delta').
\end{align*}
Similarly, $\mathbb{E}_{\mathcal{D}_1}[F(\hat{x}) - F(x^*) ] \ge 5\epsilon (p_{err}^1 - \delta')$, thus $\mathbb{E}_{\mathcal{D}}[F(\hat{x}) - F(x^*) ] \ge 5\epsilon (p_{err} - \delta')$.
The guarantee $\mathbb{E}[F(\hat{x}) - F(x^*)] \le \epsilon$ implies $5\epsilon (p_{err} - 0.05) \le \epsilon$, which leads to $p_{err} \le 0.2 + 0.05 = 0.25$.
Thus, the algorithm must identify the correct hypothesis with probability at least $0.75$.

\paragraph{Step 3: Information-theoretic lower bound.}

Let $P$ and $Q$ denote the probability distributions of the entire sequence of oracle responses observed by the algorithm when the instance is drawn from $\mathcal{D}_0$ and $\mathcal{D}_1$, respectively. Let $P'$ and $Q'$ denote the distributions of the underlying unclipped Gaussian vectors $z = (z_1, \dots, z_n)$ under $\mathcal{D}_0$ and $\mathcal{D}_1$. Since the observed gradients $-b_i$ are deterministic functions of $z_i$ (specifically, $b_i = \text{clip}(z_i)$), the data processing inequality for KL divergences implies:
\begin{equation*}
    D_{KL}(P \;\|\; Q) \le D_{KL}(P' \;\|\; Q').
\end{equation*}
The variables $z_i$ are independent, and $z_i \sim \mathcal{N}(\Delta_i, \sigma^2)$ under $\mathcal{D}_0$, while $z_i \sim \mathcal{N}(-\Delta_i, \sigma^2)$ under $\mathcal{D}_1$. 
The KL divergence between two univariate Gaussians with means $\mu_1, \mu_2$ and common variance $\sigma^2$ is $\tfrac12{(\mu_1 - \mu_2)^2}/{\sigma^2}$. Here, the separation is $2\Delta_i$ and $\sigma^2 = G^2/4$, so the KL divergence for $z_i$ is upper bounded by ${8\Delta_i^2}/{G^2}$.
Since the gradient oracles are exact and the objectives are linear, repeated queries to the same component yield identical results. Thus, information is gained only from the set of distinct components queried. Let $q_i$ be the probability (under $\mathcal{D}_0$) that component $i$ is queried at least once by the algorithm. By the chain rule for KL divergence (and linearity of expectation), the total divergence is bounded by:
\begin{equation} \label{eq:lb-kl_bound}
    D_{KL}(P \;\|\; Q) 
    \le D_{KL}(P' \;\|\; Q')
    \le \sum_{i=1}^n q_i \frac{8\Delta_i^2}{G^2}.
\end{equation}
To relate the divergence to the error probability $p_{err}$, we use Le Cam's argument, which states that for any hypothesis testing problem with uniform prior, 
$
    p_{err} \ge \tfrac12 \left( 1 - TV(P, Q) \right),
$
where $TV(P, Q)$ is the total variation distance. Given our requirement that $p_{err} \le \tfrac14$, we must have:
\begin{equation*}
    \tfrac14 \ge \tfrac12 \left( 1 - TV(P, Q) \right) 
    \implies 
    TV(P, Q) \ge \tfrac12.
\end{equation*}
By Pinsker's inequality, 
$
    TV(P, Q) \le \sqrt{\tfrac12 D_{KL}(P \;\|\; Q)},
$
which yields the required lower bound on the KL divergence:
\begin{equation*}
    \tfrac12 \le \sqrt{\tfrac12 D_{KL}(P \;\|\; Q)} 
    \implies D_{KL}(P \;\|\; Q) \ge \tfrac12
    .
\end{equation*}
Substituting the upper bound from \cref{eq:lb-kl_bound} yields the condition:
\begin{equation} \label{eq:lb-condition}
    \sum_{i=1}^n q_i \Delta_i^2 \ge \frac{G^2}{16}.
\end{equation}

\paragraph{Step 4: Optimizing the lower bound.}

From \cref{eq:lb-condition}, any algorithm solving the problem must satisfy the information constraint $\sum_{i=1}^n q_i \Delta_i^2 \ge {G^2}/{16}$. Since the total expected query cost $\mathcal{C}$ (under the mixture distribution $\mathcal{D}$) is lower bounded by half times the expected cost under $\mathcal{D}_0$, we also have $\mathcal{C} \ge \tfrac12 \sum_{i=1}^n q_i c_i$.
%
%
For a fixed set of parameters $\Delta_i$, we can lower bound the algorithm's expected cost in terms of the ``information density'' $\rho_i = \Delta_i^2/c_i$ of the components; indeed, 
\begin{align*}
    \frac{G^2}{16} \le \sum_{i=1}^n q_i \Delta_i^2 = \sum_{i=1}^n (q_i c_i) \rho_i \le \left(\sum_{i=1}^n q_i c_i\right) \max_i \rho_i
    \implies \mathcal{C} \ge \frac{G^2}{32 \max_i \rho_i}
    .
\end{align*}
To maximize this lower bound, the adversary minimizes $\max \rho_i$ subject to the constraints $\frac{1}{n}\sum_{i=1}^n \Delta_i = 20\epsilon$ and $\Delta_i \leq G/2$ for all $i$ (recall \cref{eq:lb-delta-conditions}). The minimum is achieved when the ratios $\rho_i$ are equalized across the active components.
Thus, we set $\Delta_i = \alpha \sqrt{c_i}$ for components $i$ in an ``active'' subset $S^* \subseteq [n]$, and $\Delta_i = 0$ otherwise.
%
The scaling factor $\alpha$ is determined using the constraint $\frac{1}{n}\sum_{i=1}^n \Delta_i = 20\epsilon$:
\begin{equation} \label{eq:alpha_val}
    \frac{1}{n} \sum_{i \in S^*} (\alpha \sqrt{c_i}) = 20\epsilon \implies \alpha = \frac{20 n \epsilon}{\sum_{i \in S^*} \sqrt{c_i}}.
\end{equation}
To certify the additional condition $\Delta_i \le G/2$ for all $i$, we require for all $i \in S^*$:
\begin{align*}
    \alpha \sqrt{c_i} \le \frac{G}{2} 
    \iff \frac{20 n \epsilon \sqrt{c_i}}{\sum_{j \in S^*} \sqrt{c_j}} \le \frac{G}{2} 
    \iff \frac{ \max_{i \in S^*} \sqrt{c_i}}{\frac1n \sum_{j \in S^*} \sqrt{c_j}} \le \frac{G}{40 \epsilon}.
\end{align*}
This is the precisely the uniformity condition given in the theorem statement. The subset $S^*$ is defined as the largest set of cheapest components satisfying this condition.
Finally, substituting $\alpha$ back into the cost lower bound, we conclude:
\begin{align*}
    \mathcal{C} 
    \ge \frac{G^2}{32} \left( \frac{\sum_{i \in S^*} \sqrt{c_i}}{20 n \epsilon} \right)^2 
    = \frac{G^2}{12800 \epsilon^2} \left( \frac{1}{n} \sum_{i \in S^*} \sqrt{c_i} \right)^2.
\end{align*}
\end{proof}

%% file: neurips/appendix/theory_vs_practice.tex
\newpage
\section{Extensions and Generality of the Theoretical Analysis}
\label{app:extensions}

In this section, we discuss the generality of the theory in Sections~\ref{sec:importance_sampling} and~\ref{app:subset_selection}. In particular, for the main paper we assume convexity in order to derive meaningful theoretical statements. Given that our empirical setting does not optimize a convex loss function, we address the differences here. Importantly, we emphasize that our theory does not strictly rely on convexity. 

\subsection{Differences between Theory and Practice}
In the GRPO experiments, our training objective is not convex. Additionally, we do not use the averaged iterate but instead the last iterate for reporting final error.  Furthermore, in GRPO, the objective changes every time a new set of rollouts is collected. As such, one can think of our implementation as many rounds of cost-aware learning. With every new set of rollouts, we define a new sampling distribution over the newly generated pool. The policy is then optimized by iterating over mini-batches sampled from this reweighted pool, ensuring the cost-aware distribution actively shapes the variance of the gradient updates at each step.

\subsection{Generality of the Cost Objective in Non-Convex Settings}
\label{app:nonconvex_generality}

The mathematical optimality of $p^*$ is not strictly limited to the convex setting. Its derivation depends solely on the relationship between iteration complexity ($T$) and estimator variance ($S(p)$). Across various smoothness conditions, the number of iterations $T$ required to reach an $\epsilon$-stationary point is fundamentally bottlenecked by the stochastic gradient variance. Specifically:

\begin{itemize}
    \item \textbf{Non-Convex SGD:} Reaching an $\epsilon$-stationary point requires an iteration complexity of $T \propto \frac{S(p)}{\epsilon^2}$ \citep{ghadimi2013stochastic}.
    \item \textbf{Adaptive Optimizers (Adam):} Recent theoretical analyses of Adam-type algorithms demonstrate a convergence rate of $\tilde{\mathcal{O}}\left( \frac{\sqrt{S(p)}}{\sqrt{T}} \right)$ \citep{chen2018convergence}, implying $T \propto \frac{S(p)}{\epsilon^2}$ (modulo sub-polynomial logarithmic factors).
    \item \textbf{Generalized Smoothness:} LLM loss landscapes exhibit generalized smoothness, where local curvature grows proportionally with the gradient norm. Optimization theory establishes that gradient clipping and adaptive methods guarantee convergence in this regime \citep{zhang2020gradientclippingacceleratestraining}. Crucially, this theory establishes that the dominant stochastic term in the iteration complexity explicitly retains the $\mathcal{O}\left(\frac{S(p)}{\epsilon^2}\right)$ dependence on the variance $S(p)$.
\end{itemize}
Because the total expected training cost $\mathcal{K}$ is the product of iterations and the expected per-step cost ($T \cdot C(p)$), the objective to minimize reduces to:
$$J(p) = S(p)C(p)$$
Since this objective remains structurally identical across these diverse optimization settings, minimizing the variance-cost product via $p^*$ continues to be the precise optimal strategy. Thus, the optimal distribution $p_i^* \propto G_i / \sqrt{c_i}$ derived in Theorem~\ref{thm:p_star} remains justified in non-convex regimes.

\subsection{Average Iterate ($\bar{x}_T$) vs. Last Iterate ($x_T$)}
While Algorithm~\ref{alg:cost_weighted_sgd} returns the average iterate $\bar{x}_T$ following standard conventions in the SGD literature, the properties of the last iterate $x_T$ have also been extensively studied \citep{shamir2013stochastic, pmlr-v99-jain19a, pmlr-v99-harvey19a}. As shown by \citet{pmlr-v99-harvey19a}, for a fixed step size, the last-iterate of SGD is a factor of $\log(T)$ worse than the average. While the average iterate generally achieves superior performance—sometimes strictly so—specific step-size modification schemes \citep{pmlr-v99-jain19a} allow the expected suboptimality of the last iterate to enjoy the same convergence guarantees. Consequently, sampling with the optimal distribution $p^*$ derived in Theorem~\ref{thm:p_star} provides similar theoretical guarantees for the last iterate as we establish for the averaged iterate.

%% file: neurips/appendix/experiment_details.tex
\newpage
\section{Experimental Details and Ablations}
\label{app:experiment_details}
\subsection{Experimental Details}
To ensure the reproducibility of our experiments, we provide a comprehensive overview of the hyperparameter configurations and the computational framework employed. Our implementation is built upon the Verl framework, utilizing vLLM \citep{kwon2023efficient} for efficient inference and Flash-Attention \citep{dao2022flashattention} for optimized memory throughput.

\subsection{Implementation}
The primary modifications to the Verl codebase involve the construction of mini-batches for the policy gradient update and the integration of importance sampling. To maintain training stability, we re-center the importance weights around $1$. This normalization allows us to maintain a consistent gradient norm across experiments. Crucially, this simply multiplies the training objective by a constant.

\subsection{Importance Weighting with GRPO}
\label{app:iw_grpo}

Let $P(Q)$ be the distribution over prompts and $\pi_{\theta_\text{old}}(q)$ be the distribution over responses from policy $\pi_{\theta_\text{old}}$ given prompt $q$. The objective contribution of a prompt $q$ and response $o$ is as follows:
$$
\ell(o;\theta, q) = \frac{1}{|o|}
\sum_{t=1}^{|o|}
\min\left(
\frac{\pi_\theta(o_{t} \mid q, o_{<t})}{\pi_{\theta_{\mathrm{old}}}(o_{t} \mid q, o_{<t})} A^{(t)},
\mathrm{clip}\left(\frac{\pi_\theta(o_{t} \mid q, o_{<t})}{\pi_{\theta_{\mathrm{old}}}(o_{t} \mid q, o_{<t})}, 1-\epsilon, 1+\epsilon\right) A^{(t)}
\right)
 - \beta\mathbb{D}_{\mathrm{KL}}(\pi_\theta \Vert \pi_\text{ref}).
$$
The GRPO objective for parameters $\theta$ is:
\begin{equation}\label{eq:grpo}
\begin{aligned}
\mathcal{J}_{\mathrm{GRPO}}(\theta)
= \;&
\mathbb{E}_{\substack{
q \sim P(Q) \\
\{o_i\}_{i=1}^M \sim \pi_{\theta_{\mathrm{old}}}(q)
}}
\Bigg[
\frac{1}{M}
\sum_{g=1}^M
\,\ell\!\left(o_{i};\theta\right)\Bigg].
\end{aligned}
\end{equation}
Given $n$ prompts and $M$ rollouts per prompt, and letting $o_{i, g}$ be the $g$'th rollout from the $i$'th prompt, the empirical objective is:
\begin{equation}
\hat{\mathcal{J}}_{\mathrm{GRPO}}(\theta)
=
\frac{1}{nM}
\sum_{i=1}^n
\sum_{g=1}^M
\ell\!\left(o_{i,g};\theta, q_i \right).
\end{equation}
In effect, the empirical GRPO objective is a finite sum objective over prompts and rollouts. If we sample according to another distribution $p \in \Delta^{nM}$, where we sample rollout pair $(q_i, o_{i, g})$ with probability $p_{(i,  g)}$, the importance-weighted objective is:
\begin{equation}
\hat{\mathcal{J}}^{(p)}_{\mathrm{GRPO}}(\theta)
=
\frac{1}{nM}
\sum_{i=1}^n
\sum_{g=1}^M
\frac{1}{p_{(i, g)}}
\,\ell\!\left(o_{i,g};\theta, q_i \right).
\end{equation}

\newpage
\subsection{Hyperparameters}
Table~\ref{tab:hyperparameters_qwen1_5} summarizes the key hyperparameters used for both the GRPO baseline and our GRPO + ZVF method for the 1.5B experiments. For the Qwen3-8B experiments, we detail hyperparameters in the Table~\ref{tab:hyperparameters_qwen3}.  We report Scale-RL hyperparameters in Table~\ref{tab:hyperparameters_scalerl}, largely influenced by the findings of~\citep{khatri2025art}.

\begin{table*}[htbp]
\centering
\renewcommand{\arraystretch}{1.2}
\begin{tabular}{l l}
\hline
\textbf{Hyperparameter} & \textbf{Value} \\
\hline
Base Model & \texttt{Qwen/Qwen2.5-Math-1.5B-Instruct} \\
Algorithm & GRPO \\
Learning Rate & $1 \times 10^{-6}$ \\
Weight Decay & 0.1 \\
Train Batch Size & 256 \\
Rollout Samples per Prompt ($n$) & 16 \\
Rollout Temperature & 1.0 \\
PPO Mini-Batch Size & 64 \\
KL Loss Coefficient & 0.001 \\
Clip Ratio Low & 0.2 \\
Clip Ratio High & 0.28 \\
\hline
\end{tabular}
\caption{Key Training Hyperparameters for 1.5B training, GRPO and GRPO + ZVF.}
\label{tab:hyperparameters_qwen1_5}
\end{table*}

\begin{table*}[htbp]
\centering
\renewcommand{\arraystretch}{1.2}
\begin{tabular}{l l}
\hline
\textbf{Hyperparameter} & \textbf{Value} \\
\hline
Base Model & \texttt{Qwen/Qwen3-4B} \\
Algorithm & GRPO \\
Learning Rate & $1 \times 10^{-6}$ \\
Total Epochs & 15 \\
Train Batch Size & 512 \\
PPO Mini-Batch Size & 256 \\
Rollout Samples per Prompt ($n$) & 8 \\
Rollout Temperature & 1.0 \\
Max Prompt Length & 1024 \\
Max Response Length & 4096 \\
KL Loss Coefficient & 0.001 \\
\hline
\end{tabular}
\caption{Key Training Hyperparameters for Qwen3-4B Training. }
\label{tab:hyperparameters_qwen3-4b}
\end{table*}

\begin{table*}[htbp]
\centering
\renewcommand{\arraystretch}{1.2}
\begin{tabular}{l l}
\hline
\textbf{Hyperparameter} & \textbf{Value} \\
\hline
Base Model & \texttt{Qwen/Qwen3-8B} \\
Algorithm & GRPO \\
Learning Rate & $1 \times 10^{-6}$ \\
Total Epochs & 15 \\
Train Batch Size & 512 \\
PPO Mini-Batch Size & 256 \\
Rollout Samples per Prompt ($n$) & 5 \\
Rollout Temperature & 1.0 \\
Max Prompt Length & 1024 \\
Max Response Length & 4096 \\
KL Loss Coefficient & 0.001 \\
\hline
\end{tabular}
\caption{Key Training Hyperparameters for Qwen3-8B Training. }
\label{tab:hyperparameters_qwen3}
\end{table*}

\begin{table*}[htbp]
\centering
\renewcommand{\arraystretch}{1.2}
\begin{tabular}{l l}
\hline
\textbf{Hyperparameter} & \textbf{Value} \\
\hline
Base Model & \texttt{Qwen/Qwen2.5-Math-1.5B-Instruct} \\
Advantage Estimator & GRPO \\
Policy Loss Type & CISPO \\
Learning Rate & $5 \times 10^{-7}$ \\
LR Warmup Steps & 100 \\
Weight Decay & 0.01 \\
Train Batch Size & 512 \\
PPO Mini-Batch Size & 128 \\
Gradient Clipping & 1.0 \\
KL Loss Coefficient & 0.001 \\
Entropy Coefficient & 0 \\
Clip Ratio Low & 0.0 \\
Clip Ratio High & 5.0 \\
Clip Ratio C & 3.0 \\
LM Head Precision & FP32 \\
\hline
\end{tabular}
\caption{Key Training Hyperparameters for ScaleRL/CISPO Optimized Script}
\label{tab:hyperparameters_scalerl}
\end{table*}

\newpage
\subsection{Computational Requirements}
All experiments were conducted using the hardware configurations standard to the Verl framework. The training and evaluation durations are as follows:
\begin{itemize}
    \item \textbf{GRPO:} Approximately 48 hours.
    \item \textbf{GRPO + ZVF:} Approximately 72 hours.
    \item \textbf{Scale RL: } Approximately 72 hours.
\end{itemize}
For the Qwen3-4B experiments, we used $8$ NVIDIA H100 GPUs per training run. The training runs required around $2$ days each. 

For the Qwen3-8B experiments, we used $8$ NVIDIA H100 GPUs per training run. The training runs required around $5$ days each. 

\newpage
\subsection{Ablation on variants of $p^*$}
\label{app:p_star_ablation}
In this section we ablate the different variants of $p^*$. In particular, we also report $p^*_\text{smooth}$ for $\alpha=0.05, \alpha=0.1$. Furthermore, we consider $p^*_\text{LEN}$, in which $(p^*_\text{LEN})_i = \frac{1}{\sqrt{c_i}}$. This sampling distribution only considers the length of sequences in its distribution. We report these variants both for GRPO and GRPO+ZVF. In particular with the GRPO setting, we see that $p^*_\text{LEN}$ does not attain very strong performance. We hypothesize that this is because many shorter responses have $0$ advantage, and hence there is little training signal. As in the main paper, we also report final performance for all methods. 

\begin{figure}[htbp]
    \centering
    \begin{subfigure}[b]{0.49\textwidth}
        \centering
        \includegraphics[width=\textwidth]{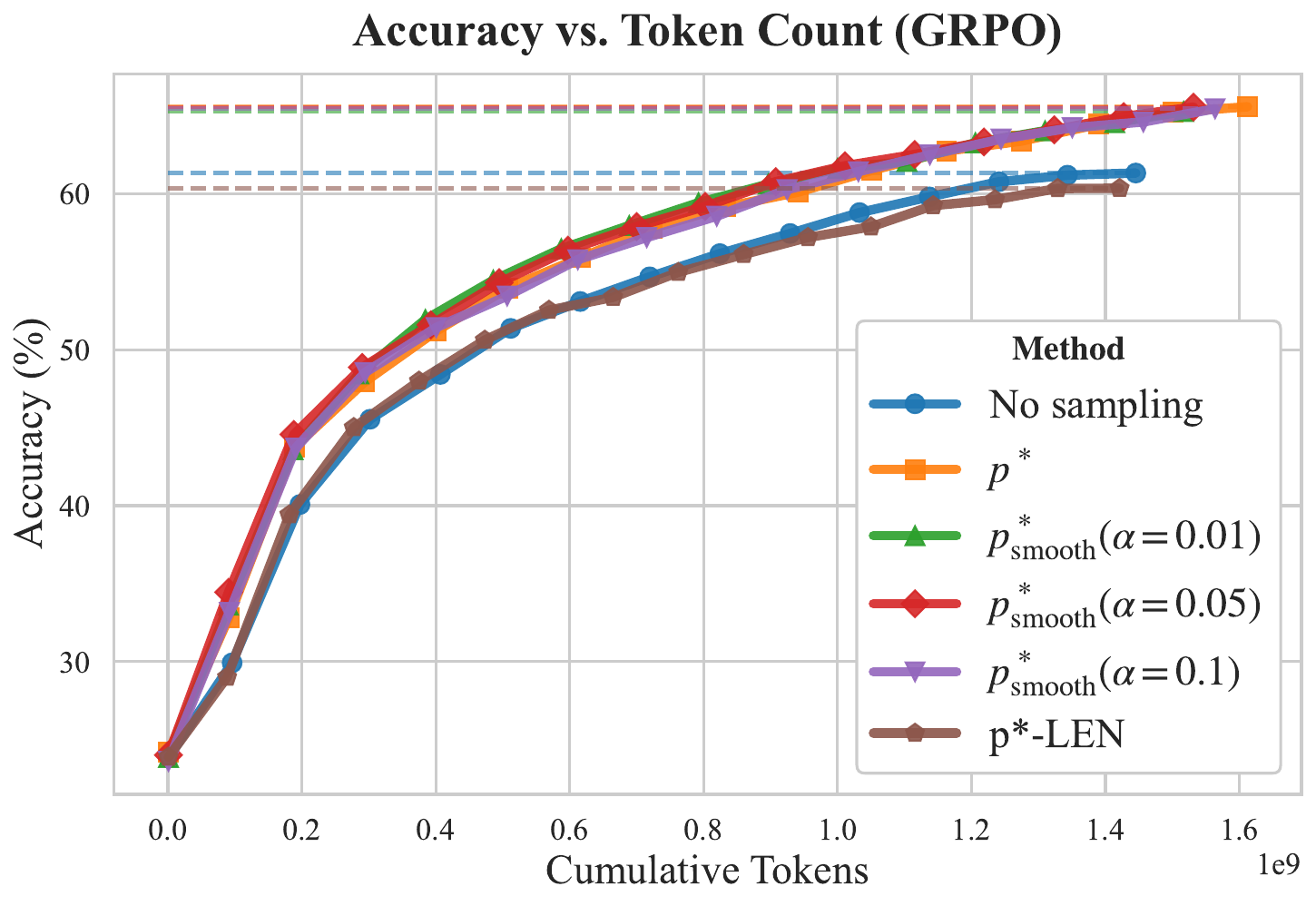}
    \end{subfigure}
    \hfill
    \begin{subfigure}[b]{0.49\textwidth}
        \centering
        \includegraphics[width=\textwidth]{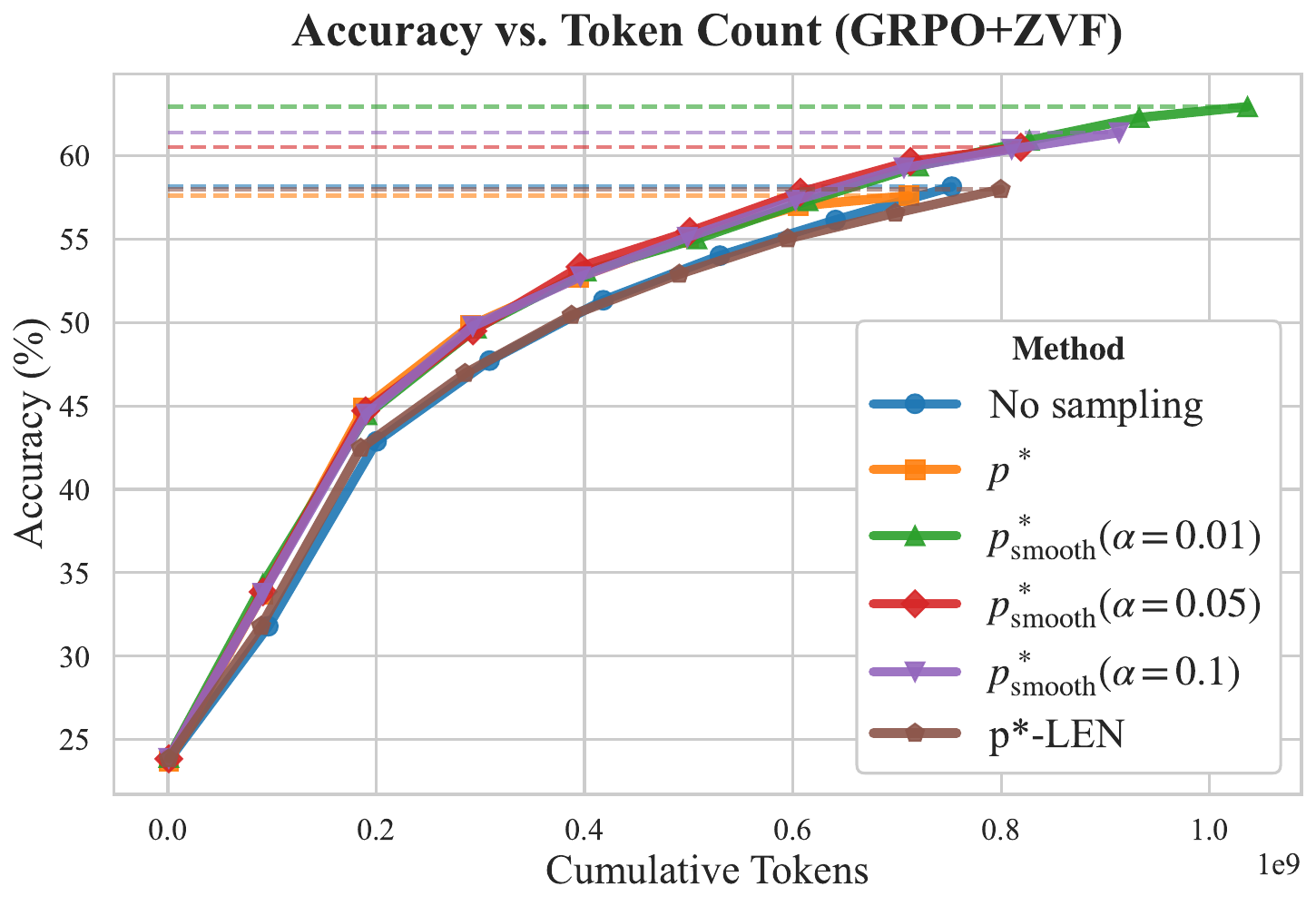}
    \end{subfigure}
    \caption{Cumulative tokens compared with AIME \texttt{pass@1/mean@32} accuracy throughout training for both GRPO and GRPO+ZVF methods.}
\end{figure}

\begin{table*}[htbp]
\centering
\begin{tabular}{llccccc}
\toprule
Setting & Method & AIME & AMC & MATH500 & GSM8K & Avg. Accuracy \\
\midrule
GRPO & No sampling & 61.3 & 64.1 & 73.2 & \textbf{86.2} & 71.2 \\
GRPO & $p^*$ & \textbf{65.6} & \textbf{71.2} & 73.2 & 86.0 & \textbf{74.0} \\
GRPO & $p^*_{\text{smooth}}(\alpha = 0.01)$ & 65.3 & 68.1 & \textbf{72.3} & 86.0 & 72.9 \\
GRPO & $p^*_{\text{smooth}}(\alpha = 0.05)$ & 65.6 & 68.3 & 72.8 & 86.1 & 73.2 \\
GRPO & $p^*_{\text{smooth}}(\alpha = 0.1)$ & 65.4 & 66.0 & 72.6 & 85.8 & 72.5 \\
GRPO & $p^*$-LEN & 60.3 & 64.8 & 73.0 & 86.0 & 71.0 \\
\midrule
ZVF & No sampling & 58.1 & 64.7 & 73.2 & 86.1 & 70.6 \\
ZVF & $p^*$ & 57.6 & 63.0 & 72.9 & 85.9 & 69.8 \\
ZVF & $p^*_{\text{smooth}}(\alpha = 0.01)$ & 62.9 & 68.0 & 73.3 & 85.8 & 72.5 \\
ZVF & $p^*_{\text{smooth}}(\alpha = 0.05)$ & 60.5 & 65.9 & 73.0 & 86.4 & 71.4 \\
ZVF & $p^*_{\text{smooth}}(\alpha = 0.1)$ & 61.4 & 68.3 & 72.9 & 86.0 & 72.2 \\
ZVF & $p^*$-LEN & 58.0 & 63.3 & 73.9 & 86.2 & 70.3 \\
\bottomrule
\end{tabular}
\caption{(Math-1.5B-Instruct, \texttt{pass@1/mean@32}): Complete results.}
\label{tab:perf_merged_complete}
\end{table*}

\begin{table*}[htbp]
\centering
\begin{tabular}{llccccc}
\toprule
Setting & Method & AIME & AMC & MATH500 & GSM8K & Avg. Accuracy \\
\midrule
GRPO & No sampling & 58.3 & 72.8 & 87.0 & 96.0 & 78.5 \\
GRPO & $p^*$ & 58.5 & 74.7 & 86.7 & 95.2 & 78.8 \\
GRPO & $p^*_{\text{smooth}}(\alpha = 0.01)$ & 58.5 & 74.0 & 86.7 & 94.9 & 78.5 \\
GRPO & $p^*_{\text{smooth}}(\alpha = 0.05)$ & 58.3 & 73.6 & 86.6 & 95.6 & 78.5 \\
\bottomrule
\end{tabular}
\caption{(8B-Base, \texttt{pass@1/mean@32}): Complete results.}
\label{tab:8b_complete}
\end{table*}

\newpage
\subsection{Full CISPO results}
\label{app:cispo_results}
We report performance across benchmarks. The CISPO curves are reported in Figure~\ref{fig:qwen1_5_scale_rl_all} and the best checkpoint performance is reported in Table~\ref{tab:perf_cispo}.

\begin{figure}[htbp]
    \centering
    \includegraphics[width=0.6\linewidth]{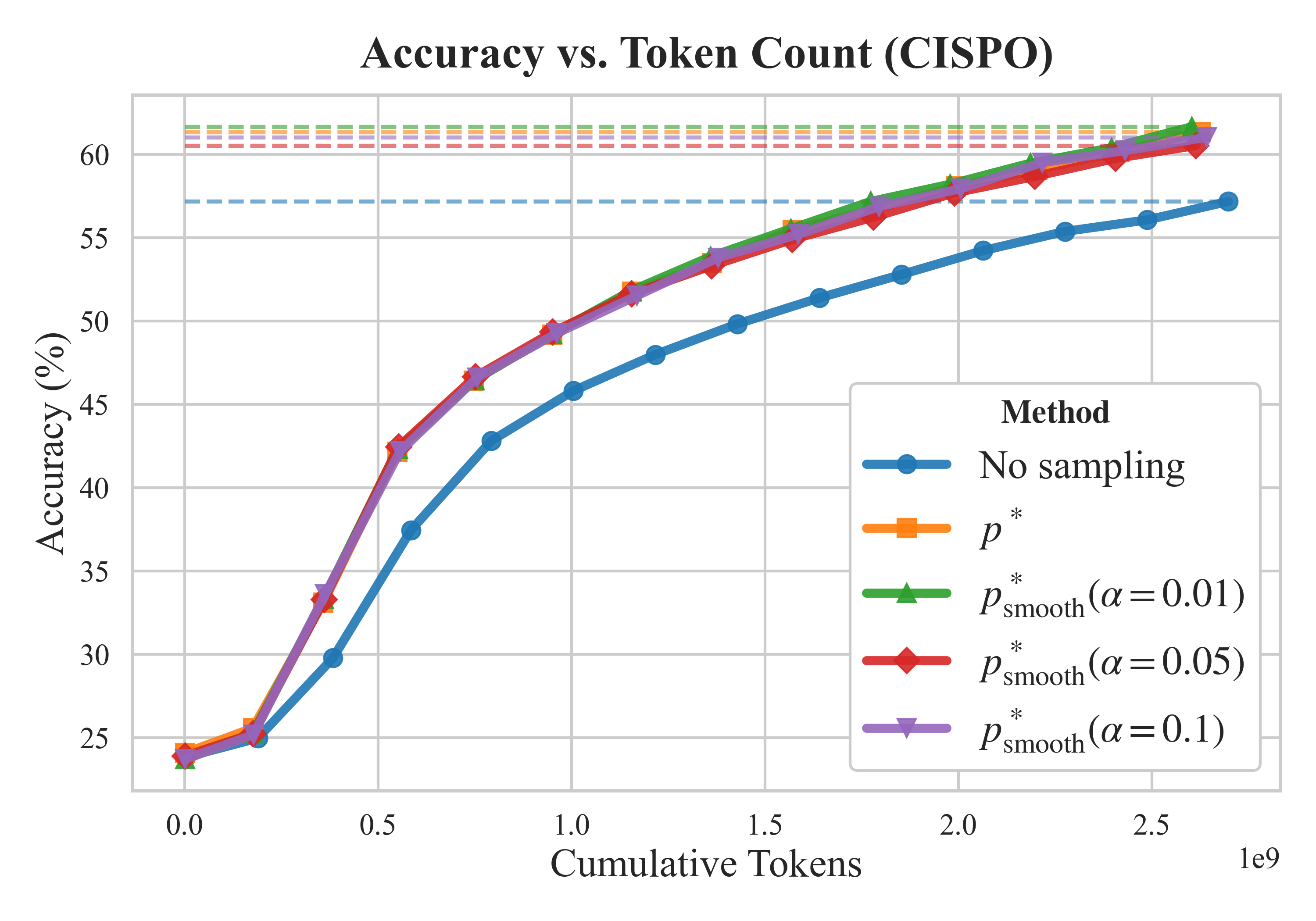}
    \caption{Full CISPO objective results for Qwen2.5-Math-1.5B-Instruct on AIME. We see robustness to smoothed cost-aware learning.}
    \label{fig:qwen1_5_scale_rl_all}
\end{figure}

\begin{table*}[htbp]
\centering
\begin{tabular}{llccccc}
\toprule
Setting & Method & AIME & AMC & MATH500 & GSM8K & Avg. Accuracy \\
\midrule
GRPO & No sampling & 57.2 & 61.8 & 73.0 & 85.9 & 69.5 \\
GRPO & $p^*$ & 61.3 & 65.7 & 73.6 & 86.2 & 71.7 \\
GRPO & $p^*_{\text{smooth}}(\alpha = 0.01)$ & 61.7 & 65.4 & 73.5 & 86.0 & 71.6 \\
GRPO & $p^*_{\text{smooth}}(\alpha = 0.05)$ & 60.5 & 66.0 & 73.0 & 86.1 & 71.4 \\
GRPO & $p^*_{\text{smooth}}(\alpha = 0.1)$ & 61.0 & 65.2 & 73.6 & 86.0 & 71.5 \\
\bottomrule
\end{tabular}
\caption{(Math-1.5B-Instruct, \texttt{pass@1/mean@32}): Complete results for CISPO.}
\label{tab:perf_cispo}
\end{table*}

\newpage
\subsection{Full sub-optimality results}
\label{app:suboptimality}

In the main paper, we report the sub-optimality metrics for the 8B and 1.5B models under the GRPO setting. Here, we also include the sub-optimality metrics for the GRPO+ZVF setting and the CISPO settings in Figures~\ref{fig:qwen1_5_zvf_subopt} and~\ref{fig:qwen1_5_scale_rl_subopt} respectively.

\begin{figure}[htbp]
    \centering
    \begin{minipage}{0.48\textwidth}
        \centering
        \includegraphics[width=\linewidth]{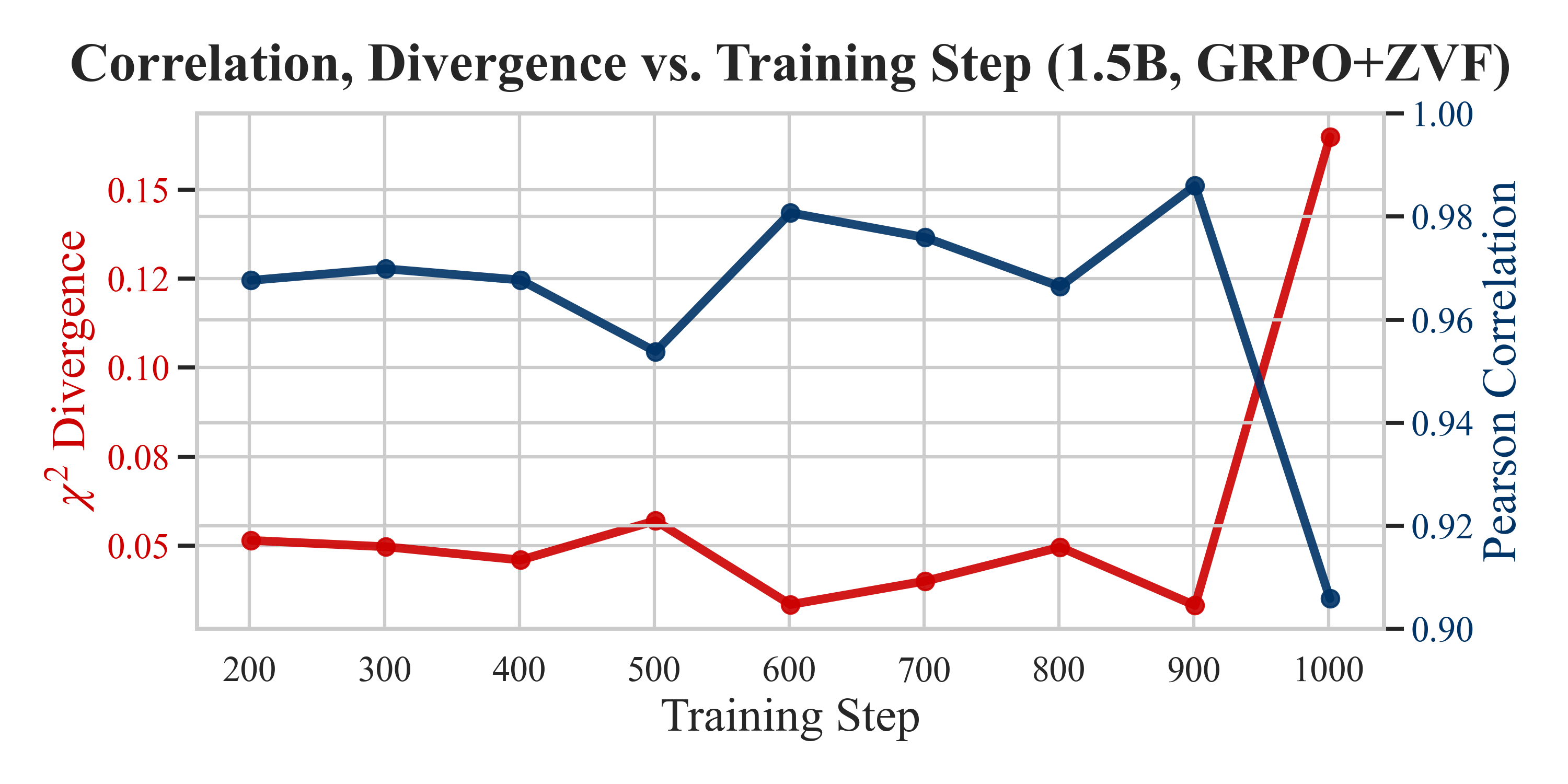}
        \caption{Sub-optimality metrics for 1.5B GRPO+ZVF training run.}
        \label{fig:qwen1_5_zvf_subopt}
    \end{minipage}
    \hfill 
    \begin{minipage}{0.48\textwidth}
        \centering
        \includegraphics[width=\linewidth]{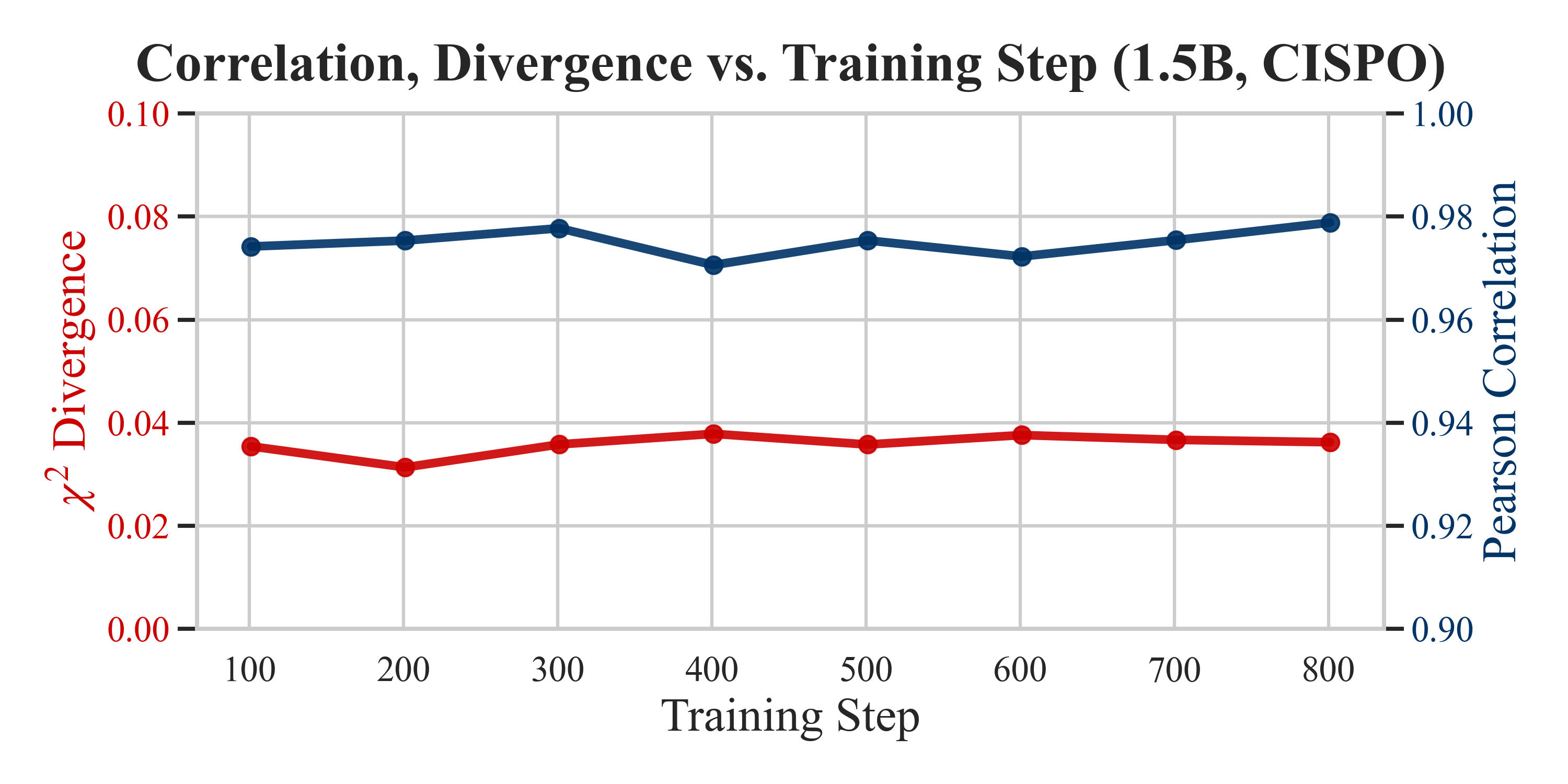}
        \caption{Sub-optimality metrics for 1.5B CISPO training objective run. }
        \label{fig:qwen1_5_scale_rl_subopt}
    \end{minipage}
\end{figure}

\begin{figure}[htbp]
    \centering
    \includegraphics[width=0.5\linewidth]{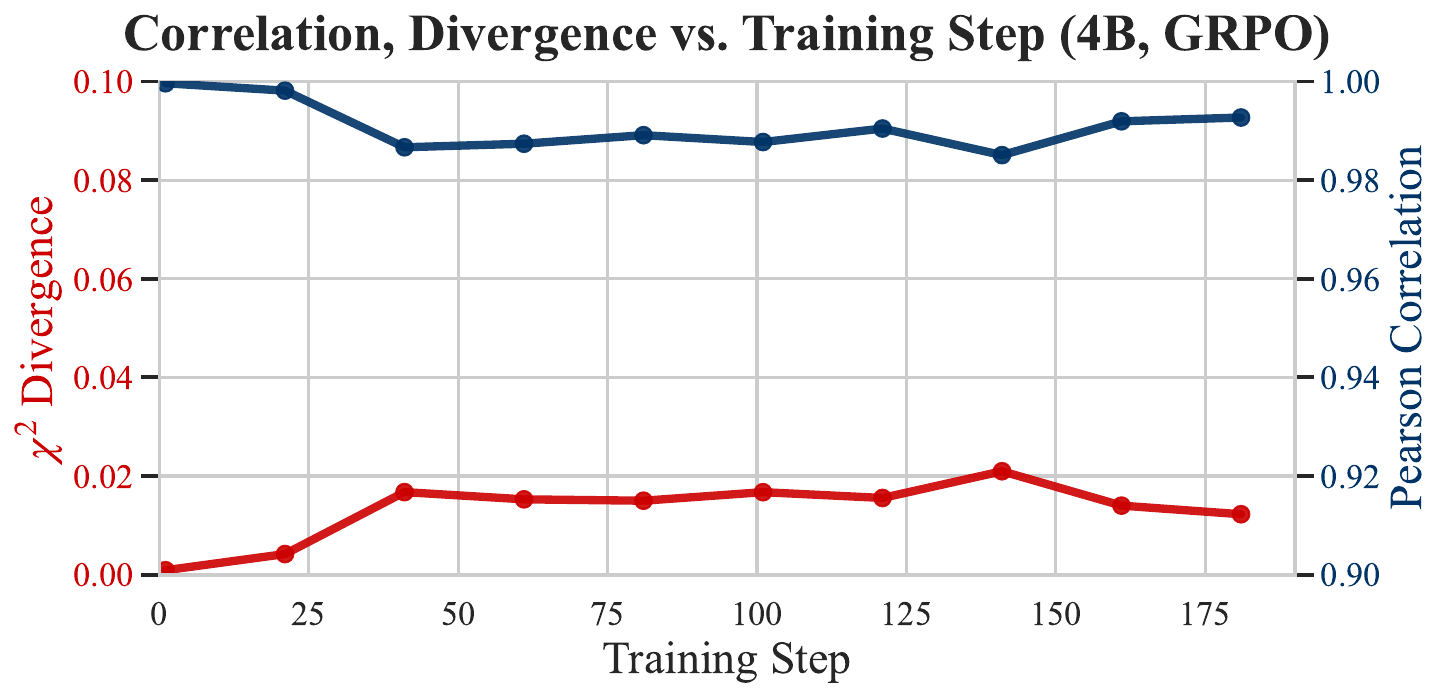}
    \caption{Sub-optimality metrics for 4B run.}
    \label{fig:4b_correlation}
\end{figure}

\newpage
\subsection{Additional results for Qwen2.5-Math-1.5B-Instruct}
We report AMC performance as a function of cumulative tokens in Figure~\ref{fig:qwen_1_5_amc_results}. 
\begin{figure*}[htbp]
    \centering
    \begin{subfigure}[b]{0.49\textwidth}
        \centering
        \includegraphics[width=\textwidth]{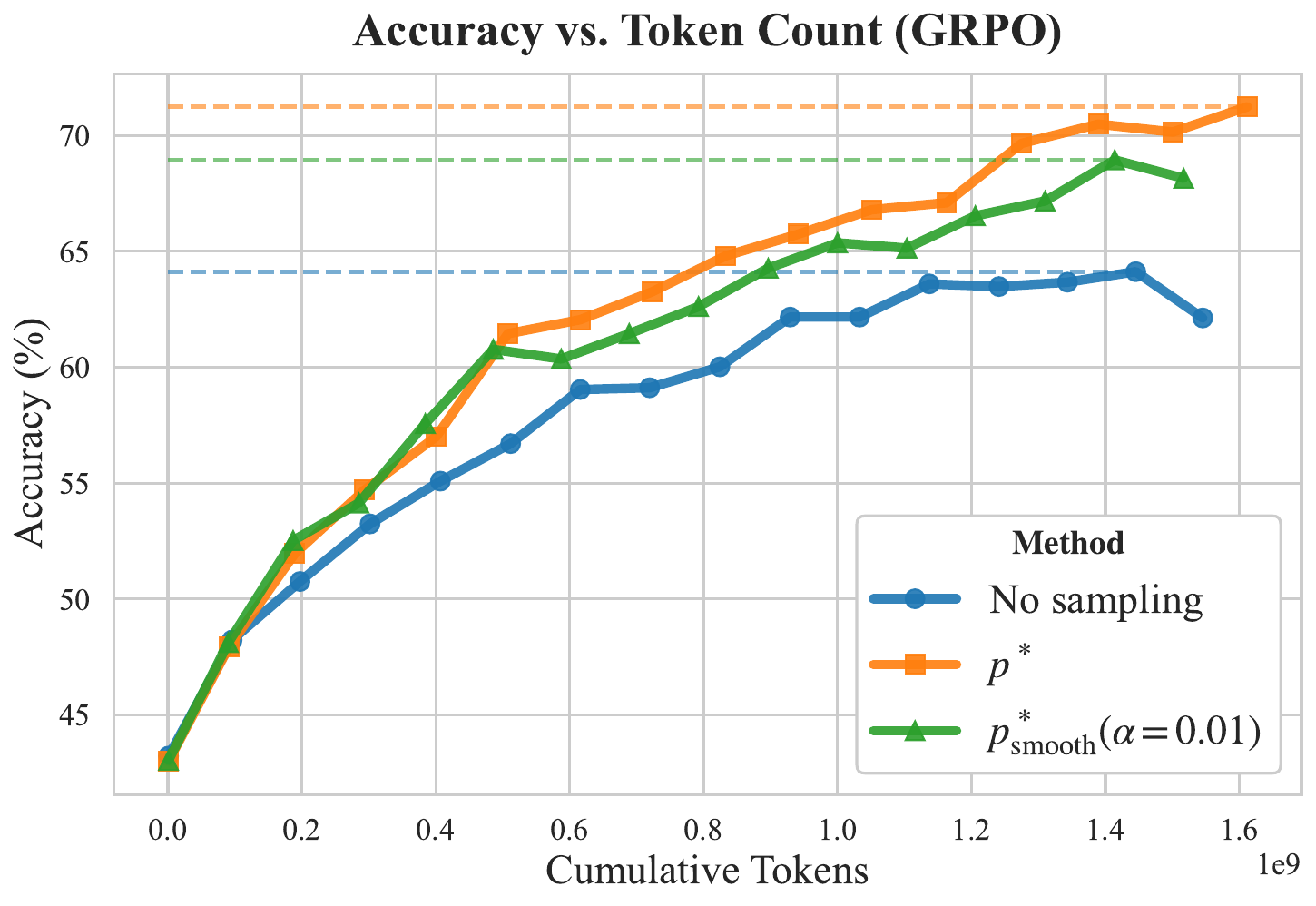}
    \end{subfigure}
    \hfill
    \begin{subfigure}[b]{0.49\textwidth}
        \centering
        \includegraphics[width=\textwidth]{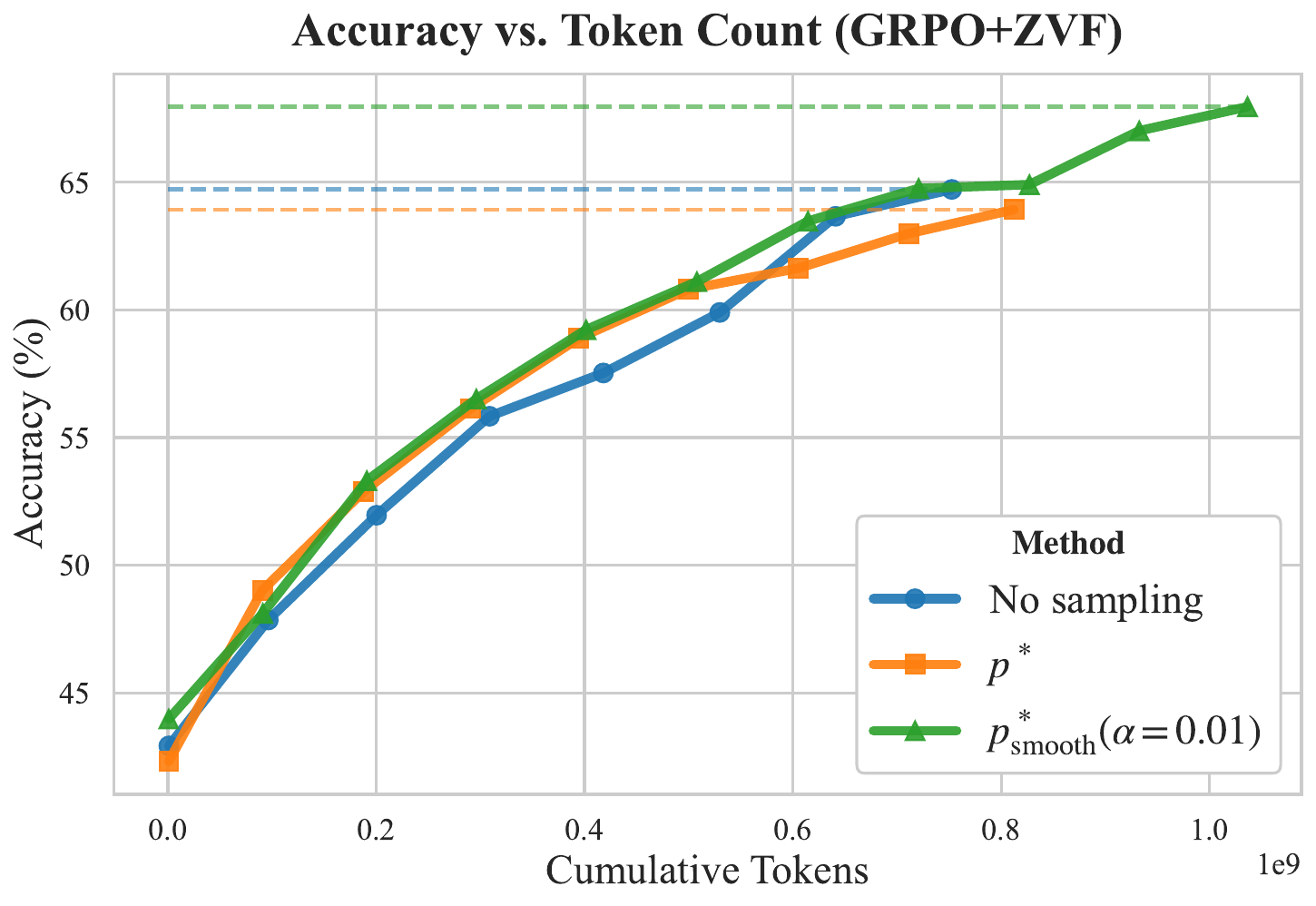}
    \end{subfigure}
    \caption{Cumulative tokens compared with AMC \texttt{pass@1/mean@32} accuracy throughout training for both GRPO and GRPO+ZVF settings for the 1.5B model. }
    \label{fig:qwen_1_5_amc_results}
\end{figure*}



\newpage 
\subsection{Additional results for Qwen3-4B}

We report results for Qwen3-4B on AIME and AMC below. We plot up until peak accuracy is attained.  We also report smoothed variants.

\begin{figure}[htbp]
    \centering
    \includegraphics[width=0.75 \linewidth]{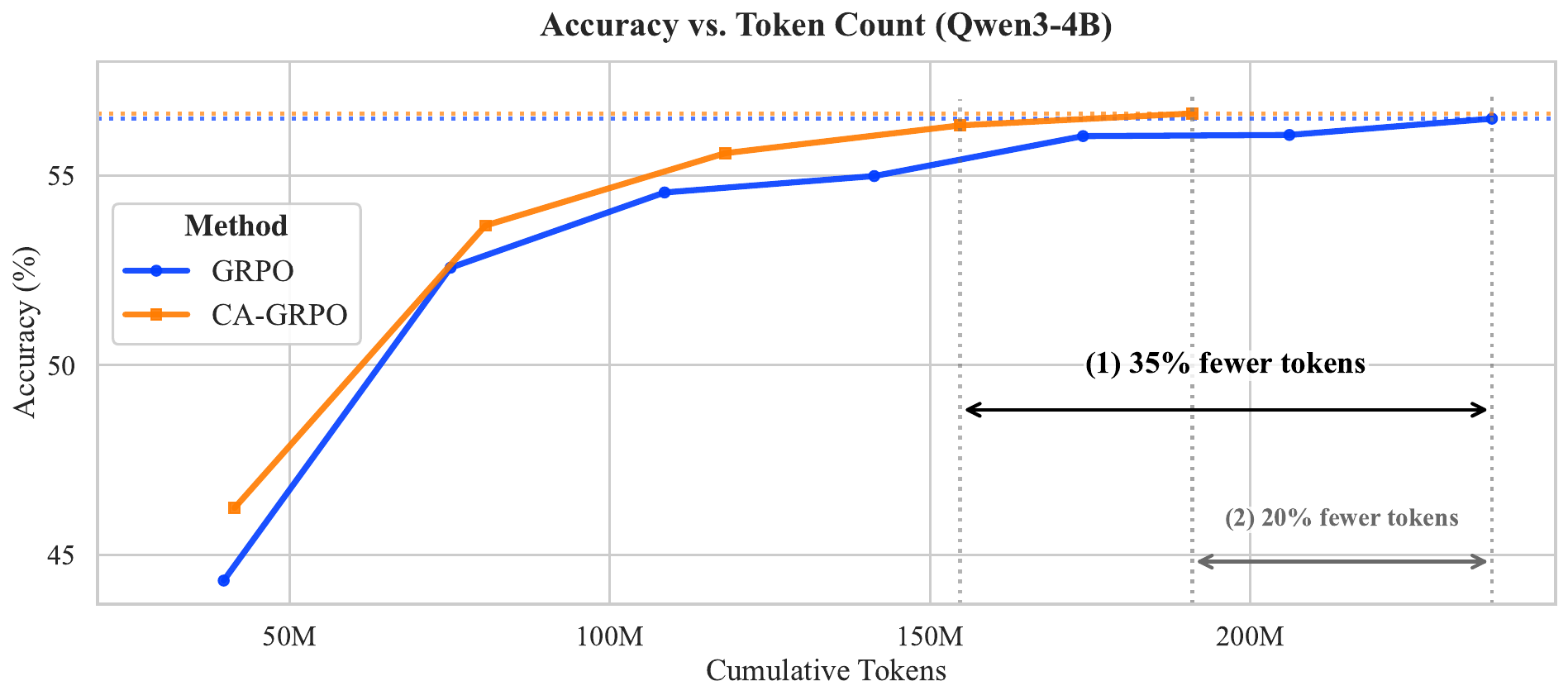}
    \caption{Qwen3-4B on AIME.}
    \label{fig:qwen3_4b}
\end{figure}

\begin{figure}[htbp]
    \centering
    \includegraphics[width=0.75\linewidth]{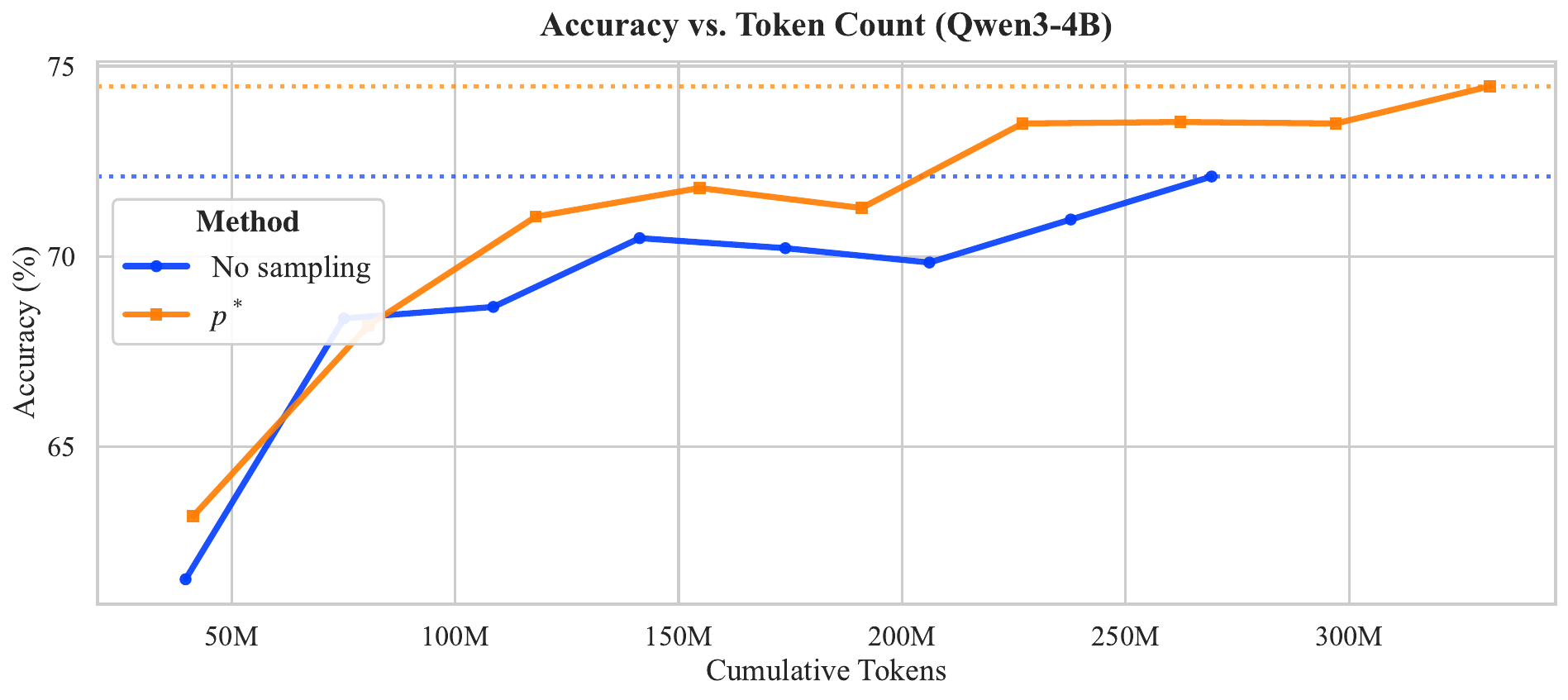}
    \caption{Qwen3-4B on AMC.}
    \label{fig:qwen3-4b-amc}
\end{figure}

\begin{figure}[htbp]
    \centering
    \includegraphics[width=0.75\linewidth]{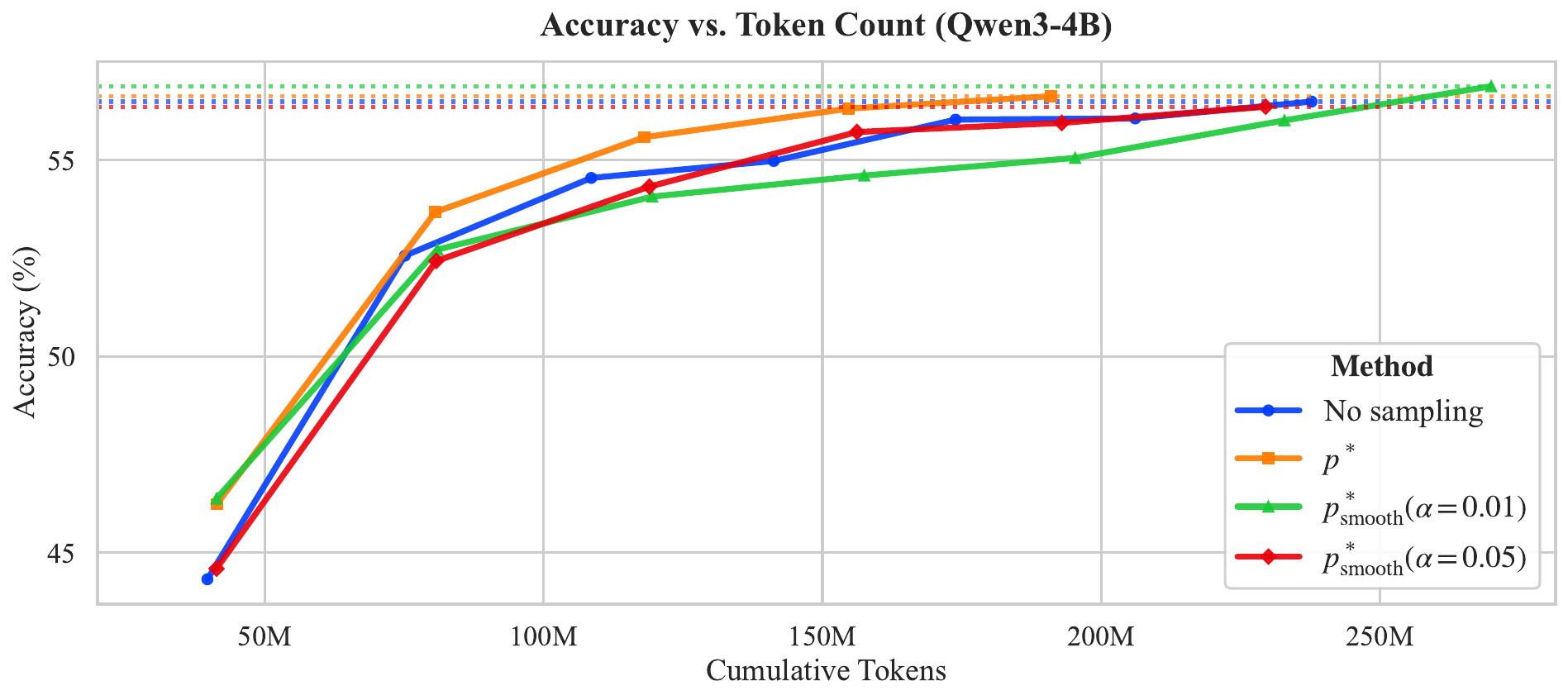}
    \caption{Qwen3-4B and smoothed variants on AIME.}
    \label{fig:qwen3-4b-all}
\end{figure}

\newpage
\subsection{Additional results for Qwen3-8B}
\label{app:8b_complete}
We present the results for the smoothed variants of the 8B model on AIME in Figure~\ref{fig:qwen3-8b-aime-all}. Additionally, we plot AMC performance as a function of cumulative tokens in Figure~\ref{fig:qwen3-8b-amc}. In Table~\ref{tab:8b_complete}, we report full results on best checkpoint accuracy for smoothed variants on four benchmarks. 

\begin{figure}[htbp]
    \centering
    \includegraphics[width=0.8\linewidth]{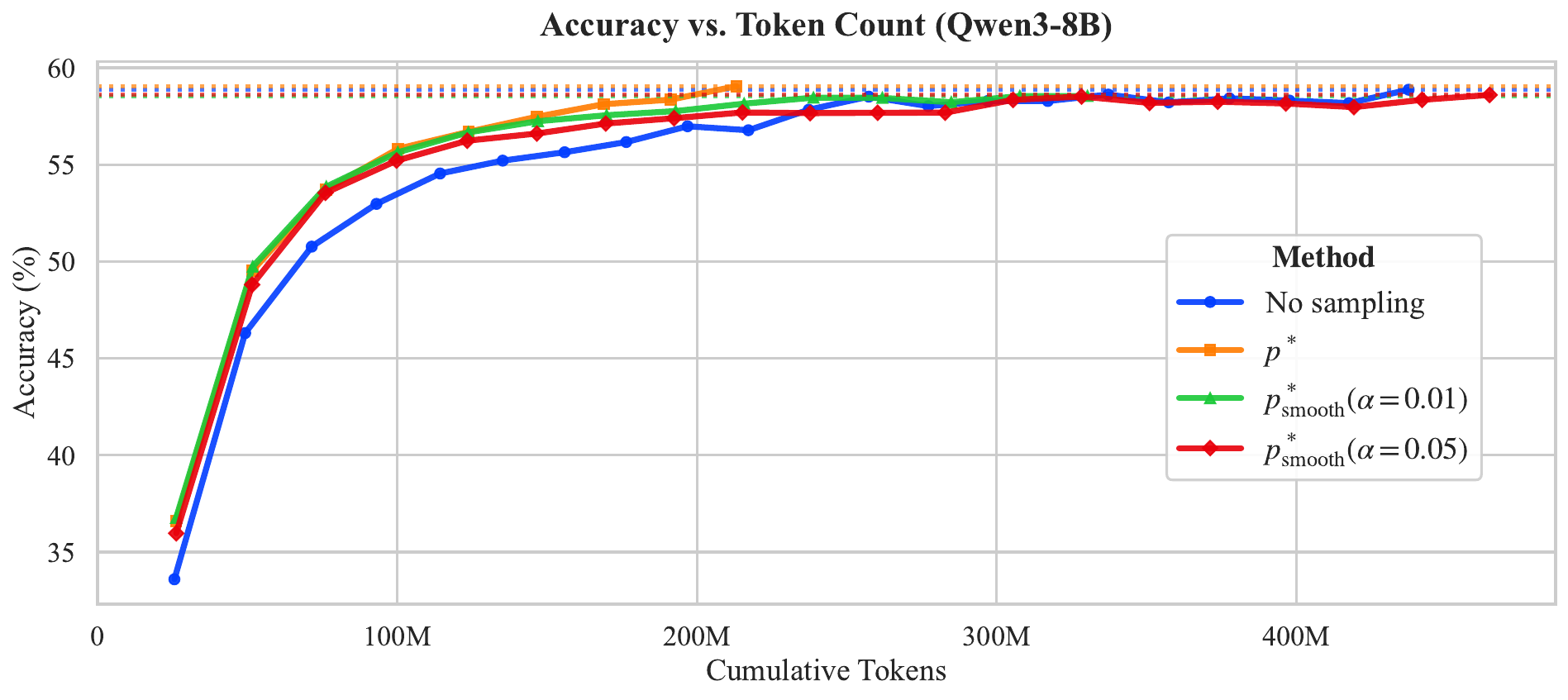}
    \caption{Qwen3-8B AIME results for all variants.}
    \label{fig:qwen3-8b-aime-all}
\end{figure}

\begin{figure}[htbp]
    \centering
    \includegraphics[width=0.8\linewidth]{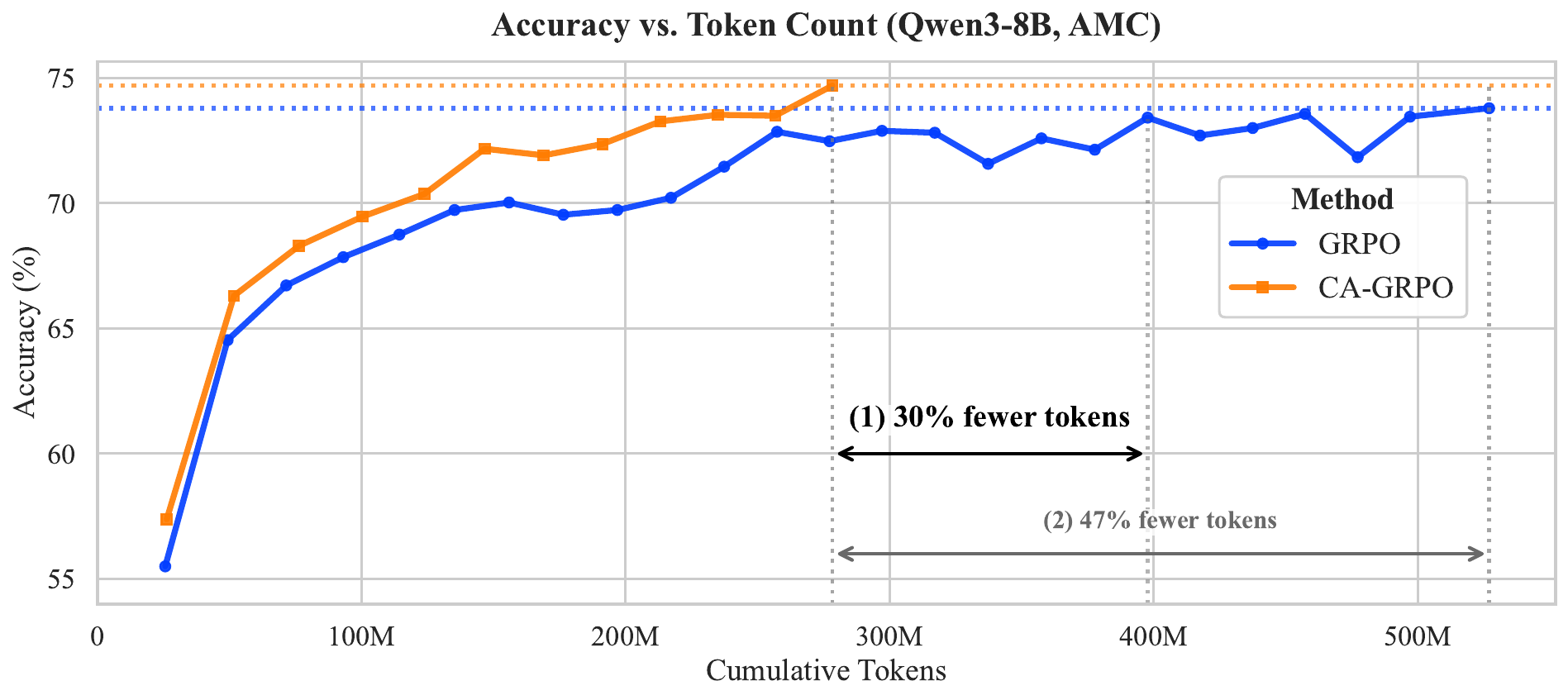}
    \caption{AMC results for Qwen3-8B Base.}
    \label{fig:qwen3-8b-amc}
\end{figure}

\newpage